\documentclass[nohyperref, accepted]{article}

\usepackage{microtype}
\usepackage{graphicx}
\usepackage{subfigure}
\usepackage{booktabs} 

\usepackage{hyperref}

\usepackage{icml2022}

\usepackage{amsmath}
\usepackage{amssymb}
\usepackage{mathtools}
\usepackage{amsthm}

\usepackage[capitalize,noabbrev]{cleveref}

\theoremstyle{plain}
\newtheorem{theorem}{Theorem}[section]

\newtheorem{lemma}[theorem]{Lemma}
\newtheorem{corollary}[theorem]{Corollary}
\theoremstyle{definition}

\newtheorem{assumption}[theorem]{Assumption}
\theoremstyle{remark}
\newtheorem{remark}[theorem]{Remark}

\usepackage[textsize=tiny]{todonotes}

\icmltitlerunning{Sparse Approximate Kernel-Based Learning}

\setcitestyle{authoryear,open={(},close={)}}
\hypersetup{
    colorlinks=true,
    linkcolor=magenta,
    citecolor =cyan,
    filecolor=magenta,      
    urlcolor=magenta,
}

\usepackage{xcolor}
\usepackage{amsmath}
\usepackage{amsfonts}
\usepackage{bm}
\usepackage{amssymb}
\usepackage{hyperref}  

\usepackage{cleveref}
\crefformat{section}{\S#2#1#3} 
\crefformat{subsection}{\S#2#1#3}
\crefformat{subsubsection}{\S#2#1#3}

\def\nn{\nonumber}

\def\Ec{\mathcal{E}}
\def\Rc{\mathcal{R}}
\def\Xc{\mathcal{X}}

\def\Hc{\mathcal{H}}

\def\Oc{\mathcal{O}}
\def\Oct{\tilde{\mathcal{O}}}
\def\Sc{\mathcal{S}}
\def\Dc{\mathcal{D}}

\def\Ic{\mathcal{I}}

\def\Rr{\mathbb{R}}
\def\Nn{\mathbb{N}}
\def\E{\mathbb{E}}
\def\Xx{\mathbb{X}}

\def\kbar{\bar{k}}
\def\mubar{\bar{\mu}}
\def\sigmabar{\bar{\sigma}}

\def\betat{\tilde{\mathcal{\beta}}}

\DeclareMathOperator{\tr}{tr}

\begin{document}

\twocolumn[
\icmltitle{Improved Convergence Rates for Sparse \\ Approximation Methods in Kernel-Based Learning}




\begin{icmlauthorlist}
\icmlauthor{Sattar Vakili}{yyy}
\icmlauthor{Jonathan Scarlett}{comp}
\icmlauthor{Da-shan Shiu}{yyy}
\icmlauthor{Alberto Bernacchia}{yyy}
\end{icmlauthorlist}

\icmlaffiliation{yyy}{MediaTek Research}
\icmlaffiliation{comp}{National University of Singapore}

\icmlcorrespondingauthor{Sattar Vakili}{sattr.vakili@mtkresearch.com}

\icmlkeywords{Machine Learning, ICML}

\vskip 0.3in
]



\printAffiliationsAndNotice{}  

\begin{abstract}

Kernel-based models such as kernel ridge regression and Gaussian processes are ubiquitous in machine learning applications for regression and optimization. It is well known that a major downside for kernel-based models is the high computational cost; given a dataset of $n$ samples, the cost grows as $\Oc(n^3)$. Existing sparse approximation methods can yield a significant reduction in the computational cost, effectively reducing the actual cost down to as low as $\Oc(n)$ in certain cases. Despite this remarkable empirical success, significant gaps remain in the existing results for the analytical bounds on the error due to approximation. In this work, we provide novel confidence intervals for the Nystr\"om method and the sparse variational Gaussian process approximation method, which we establish using novel interpretations of the approximate (surrogate) posterior variance of the models.  Our confidence intervals lead to improved performance bounds in both regression and optimization problems. 

\end{abstract}


\section{Introduction}

Kernel-based modeling is an elegant and natural technique to extend linear models to highly general nonlinear ones with great representation capacity. Powerful predictors and uncertainty estimates have made this approach very common across several application domains~\citep{Shahriari2016outofloop}. It is, however, well known that regression using standard kernel-based methods incurs a high $\Oc(n^3)$ computational cost in the number of data points $n$, which is the main limiting factor in many applications. Fortunately, sparse approximation methods have been developed for kernel-based models, which significantly reduce the computational cost. In particular, the computational cost typically reduces to $\Oc(nm^2)$, with $m\ll n$ being the number of the \emph{inducing points} that serve to summarize the model. 
Despite their great popularity in practice, significant gaps remain in the mathematical understanding of such methods' bounds on the error due to approximation, and the implications for regression and optimization problems. This challenging issue is the main focus of this work.

The literature on kernel-based models has developed under two mostly separate approaches with different terminologies~\citep[e.g., see,][]{Kanagawa2018}: kernel ridge regression (KRR) and Gaussian process (GP) models. These two approaches are based on frequentist and Bayesian interpretations, respectively. Nonetheless, they lead to the same solution for regression. In this work, we cover both terminologies and provide insights into their connections. 

In the GP approach, a prior (surrogate) GP distribution is assumed for the unknown function $f$, and Gaussian noise is assumed. The regressor is then the posterior mean (the maximum likelihood estimate) of $f$, conditioned on the data points. In contrast, in KRR approach, the regressor is the regularized least squares estimator of $f$, within a reproducing kernel Hilbert space (RKHS). For further details, see~\cite{Rasmussen2006} on the former, and~\cite{scholkopf2002learning} on the latter. 

For GPs, sparse approximation methods find a Bayesian \emph{variational} interpretation and are referred to as sparse variational Gaussian processes (SVGPs) in the literature. The SVGP method was introduced in~\cite{Titsias2009Variational} and has been studied in several works including~\cite{Hensman2013, matthews2016sparse, Burt2019Rates, rossi2021sparse}. See also~\cite{liu2020gaussian} for a recent overview. The SVGP models enjoy great popularity in practice, with implementations in \texttt{TensorFlow}~\citep{tensorflow2015-whitepaper} and {\texttt{GPflow}}~\citep{GPflow2017} libraries. 

For KRR, sparse approximation methods are interpreted as regularized least squares estimators of $f$ within a reduced rank RKHS. This enables a Nystr\"om low rank matrix approximation; the corresponding sparse approximation method is thus often referred to as the Nystr\"om method. Representative works include~\cite{williams2001using, seeger2003fast, rudi2015less, lin2018optimal}.



\textbf{Contributions and paper structure.} In Section~\ref{sec:prelim}, we introduce the notations and preliminaries on kernel-based models and their sparse approximations. 
In Section~\ref{sec:conf_intervals}, we establish a novel representation of the approximate posterior variance of a GP model in terms of \emph{i)} the error due to the low rank projection of the approximation method, \emph{ii)} the prediction error from noise-free observations in the low rank RKHS, and \emph{iii)} the effect of noise (see Theorem~\ref{the:approximatevariance}). We then use this result to derive novel confidence intervals for approximate kernel-based models (see Theorem~\ref{the:conf_subG}).  These results/tools are broadly applicable in kernel-based models, and may be of general interest.

In Section~\ref{sec:regression}, we use our confidence intervals
to provide explicit bounds on the prediction error of a regression problem using sparse approximation methods (see, Theorem~\ref{the:regression}).
In particular, we prove that the sparse approximate prediction $\mubar_n$ of $f$, using $n$ well distributed data points, converges uniformly to $f$, with a number of inducing points no larger than $\Oct(\gamma_k(n))$; here, $\gamma_k(n)$ denotes the maximal information gain of the kernel, and can be accurately bounded for specific kernels~\citep[e.g., see][]{srinivas2010gaussian, vakili2021uniform, vakili2020information}. For example, in the case of the squared exponential (SE) kernel, our results imply that with $m=\Oc\left((\log(n))^{d+2}\right)$ inducing points, where $d$ is the dimension of the input, $\mubar_n$ uniformly converges to $f$ at rate $\Oct(\frac{1}{\sqrt{n}})$. In the case of the Mat{\'e}rn family of kernels, with $m=\Oct(n^{\frac{d}{2\nu+d}})$, $\mubar_n$ uniformly converges to $f$ at a rate $\Oct(n^{\frac{-\nu}{2\nu+d}})$, where $\nu$ is the smoothness parameter of the kernel. Here and subsequently, the notations $\Oc$ and $\Oct$ are used for standard mathematical order, and that up to hiding logarithmic factors, respectively.

In Section~\ref{sec:optimization}, we consider a kernel-based black-box optimization (kernel-based bandit) problem using sparse approximation methods. Leveraging our novel confidence intervals, we prove  $\Oct(\sqrt{\gamma_k(N)N})$ regret bounds on a time horizon $N$, for an optimization algorithm based on batch observations~(see Theorem~\ref{the:optimization}). This is optimal up to logarithmic terms in the cases where lower bounds on the regret are known.
Our algorithm is an adaptation of Batched Pure Exploration (BPE) algorithm of~\cite{li2021gaussian}, in which we replace the exact GP statistics with their sparse approximations. We thus refer to this algorithm as Sparse BPE (S-BPE). 
To our knowledge, this is the first algorithm for kernel-based bandits using sparse approximation methods while achieving near-optimal regret bounds.

\subsection{Related Work}

{The celebrated work of \cite{Burt2019Rates} proved a fundamental result on bounding the Kullback-Leibler (KL) divergence between the approximate and exact GP posteriors in SVGPs \citep[see also][]{JMLR:v21:19-1015}. In particular, they showed how large the number $m$ of inducing points should be chosen in relation to the sample size $n$ in order to ensure that the expected KL divergence vanishes for large $n$. As a consequence, they showed that the approximate posterior mean and variance converge to the exact ones when $m$ is large enough. These results hold in expectation, where the expectation is taken with respect to a prior on the dataset, as well as the prior GP distribution of $f$. \cite{nieman2021contraction} proved similar bounds on the KL divergence of the approximate and exact (surrogate) GP posteriors, when $f$ is a fixed function in the RKHS. Their result also holds in expectation, where the expectation is taken only with respect to a prior on the dataset (and not a prior on $f$). }

{\cite{nieman2021contraction} also proved a result similar to our Theorem~\ref{the:regression} on the convergence rate of the approximate prediction $\mubar_n$ to the fixed $f$ in the RKHS. Neither this result nor our Theorem~\ref{the:regression} imply the other, and the comparison can be summarized as follows. Both results find the same $m=\Oct(n^{\frac{d}{2\nu+d}})$ rates for the number of inducing points in the case of the Mat\'ern kernel. Our convergence measure is stronger than theirs in the sense that our convergence guarantees are uniform (given in terms of the $L^{\infty}$ norm), while their convergence guarantees are given in terms of \emph{Hellinger} distance. Thus, their convergence result is in an average sense rather than in a uniform sense, where the average is taken with respect to a prior distribution on the data. In contrast, in the setting of Theorem~\ref{the:regression}, the data is collected based on uncertainty reduction. While these settings are different, they serve the same purpose of ensuring a nearly uniform dataset across the domain. }
In addition, we prove novel general confidence intervals based on approximate statistics (Theorem~\ref{the:conf_subG}) and near-optimal regret bounds for black-box optimization (Theorem~\ref{the:optimization}), which
are new compared to existing works.

It is also noteworthy that the bound on $m$ required for vanishing KL divergence in~\cite{Burt2019Rates, JMLR:v21:19-1015} is larger than the bound on $m$ required for the convergence of $\mubar_n$ to~$f$. 
For example, in the case of Mat{\'e}rn kernels, the results of \cite{Burt2019Rates, JMLR:v21:19-1015} require $m=\Oc(n^{\frac{2d}{2\nu-d}})$ inducing points~\citep[see][Table $1$]{JMLR:v21:19-1015}, which is a looser bound compared to $m=\Oc(n^{\frac{d}{2\nu+d}})$ in~\cite{nieman2021contraction} and in our results.

\paragraph{Selection of inducing points.} Several methods for selecting the inducing points have been employed in the existing work. \cite{Burt2019Rates} considered selecting the inducing points based on a $k$ determinantal point process ($k$-DPP) on the data.  However, the cost of sampling from an exact $k$-DPP may defy the computational gain from the sparse approximation. They thus suggested to use a \emph{Markov Chain Monte Carlo} based approximate sampling from a $k$-DPP. In practice, the inducing points are sometimes selected based on $k$-means clustering, without theoretical guarantees \citep{Hensman2013}.
Selecting the inducing points based on ridge leverage score (RLS) sampling, which captures the importance of an individual data point, comes with strong theoretical guarantees 
\citep{alaoui2014fast}.
\cite{musco2016recursive} introduced a practical recursive RLS with guarantees for low computational complexity (see Lemma~\ref{Lemma:recRLS}). \cite{chen2021fast} improved the bounds on computational complexity of recursive RLS sampling. \cite{Calandriello2019Adaptive} showed that selecting the inducing points based on approximate posterior variance enjoys the same performance guarantees as RLS. In our analysis, we use the theoretical guarantees on inducing points given in~\cite{musco2016recursive}. However, we also establish general sufficient conditions that may apply to other methods of selecting the inducing points.

\paragraph{Kernel-based black-box optimization.}  

Black box optimization of kernel-based models is often referred to as kernel-based bandits, GP bandits, or Bayesian optimization in the literature. The terminology “bandit” signifies zeroth-order observations from
$f$, in contrast to, e.g., first-order observations in
gradient descent.
Classical algorithms such as GP upper confidence bound (GP-UCB)~\citep{srinivas2010gaussian, Chowdhury2017bandit}, expected improvement (EI)~\citep{mockus1978application}, and Thompson sampling (GP-TS)~\citep{Chowdhury2017bandit} sequentially select the observation points based on a score referred to as \emph{acquisition function}. When the objective function $f$ is a fixed function in the RKHS, the best known regret bounds for these algorithms
scales as $\Oc(\gamma_k(N)\sqrt{N})$ over a time horizon $N$ (see the references above). The $\Oc(\gamma_k(N)\sqrt{N})$ scaling is not tight in general, and may even fail to be sublinear in many cases of interest, since $\gamma_k(N)$, although sublinear, may grow faster than $\sqrt{N}$. It remains an open problem whether the suboptimal regret bounds of these acquisition based algorithms is a fundamental limitation or a shortcoming of their proof~\citep[see][for the details]{vakili2021open}. 
On discrete domains, the SupKernelUCB algorithm was shown to have an $\Oct(\sqrt{\gamma_k(N)N})$ regret bound~\citep{Valko2013kernelbandit}; moreover, this result extends to compact domains under mild Lipschitz-like assumptions \cite{Janz2020SlightImprov}. This nearly matches the lower bounds proven in~\cite{Scarlett2017Lower} for the SE and Mat{\'e}rn kernels.

SupKernelUCB is not considered to be practical. Several more practical algorithms with $\Oct(\sqrt{\gamma_k(N)N})$ regret have been proposed recently: a tree-based domain-shrinking algorithm known as GP-ThreDS~\citep{salgia2021domain}, Robust Inverse Propensity Score (RIPS) for experimental design~\citep{camilleri2021high}, and BPE~\cite{li2021gaussian}. The latter two use batches of observations, with RIPS using $\Oc(\log(N))$ batches and BPE using $\Oc(\log\log(N))$ batches.

There have been several works considering sparse approximation methods in kernel-based optimization algorithms. \cite{Calandriello2019Adaptive} introduced budgeted kernelized bandit (BKB), which is an adaptation of GP-UCB by replacing the exact statistics with sparse approximate ones. They proved that BKB achieves the same $\Oct(\gamma_k(N)\sqrt{N})$ regret bound as GP-UCB, where the number of inducing points scales as $m=\Oc(\gamma_k(n))$ at time $n$. \cite{Vakili2020Scalable} introduced S-GP-TS, which is an adaptation of GP-TS by replacing the exact statistics with sparse approximate ones. They also proved that S-GP-TS achieves the same $\Oct(\gamma_k(N)\sqrt{N})$ regret bound as GP-TS, where the number of inducing points scales as $m=\Oc(n^{\frac{2d}{2\nu-d}})$. The analysis in~\cite{Vakili2020Scalable} relies on KL divergence bounds from~\cite{Burt2019Rates}, and their weaker bound on $m$ is due to the weaker scaling of the number of inducing points in~\cite{Burt2019Rates}.  

We also highlight that a recent concurrent work again attained $\Oct(\gamma_k(N)\sqrt{N})$ regret with a low time complexity via a distinct approach of resampling the same points multiple times \cite{calandriello2022scaling}.  In addition, scalable algorithms for the contextual kernel bandit setting were studied in \cite{pmlr-v151-zenati22a}.

Our confidence intervals do not apply to the analysis of the BKB or S-GP-TS algorithms, but we show that they apply to an adaptation of BPE and offer a near-optimal $\Oct(\sqrt{\gamma_k(N)N})$ regret bound, which is an $\Oc(\sqrt{\gamma_k(N)})$ improvement over BKB and S-GP-TS.

\section{Kernel-Based Models and Sparse Approximations}\label{sec:prelim}

In this section, we give an overview of the relevant background on regression using kernel-based models, as well as the sparse approximation methods.  

\subsection{Kernel-Based Regression}\label{sec:exactre}



Consider a noisy dataset $\mathcal{D}_n = \{(x_i,y_i)_{i=1}^n\}$, where $y_i = f(x_i)+\epsilon_i$, and $\epsilon_i$ is the observation noise. We use the vectorized notations $X_n= [x_1,...,x_n]^{\top}$, $Y_n= [y_1,...,y_n]^{\top}$. Consider a positive definite kernel $k(\cdot,\cdot):\Xc\times\Xc\rightarrow \Rr$, $\Xc\subset\Rr^d$, $d\in\Nn$. We have the following closed form expression for the kernel-based regressor:
\begin{equation}\label{eq:regressor}
\mu_{k,\tau, X_n, Y_n}(\cdot) = k^{\top}_{X_n}(\cdot)\left(k_{X_n,X_n}+\tau^2 I_n\right)^{-1}Y_n,
\end{equation}
where $k_{X_n}(\cdot) = [k(\cdot,x_1), \dots, k(\cdot,x_n)]^{\top}$, $k_{X_n,X_n}=[k(x_i, x_j)]_{i,j=1}^n$ is the $n\times n$ kernel matrix, $\tau$ is a free parameter, and $I_n$ is the $n\times n$ identity matrix. We use the shorthand notation $\mu_n=\mu_{k,\tau, X_n, Y_n}$ when the values of $k,\tau,X_n,Y_n$ are clear from the context.

In the KRR approach, $\mu_n$ is interpreted as the regularized least squares estimator with regularization parameter $\tau^2$:
\begin{eqnarray}
\mu_n =\arg\min_{g\in\Hc_k} \sum_{i=1}^n(y_i-g(x_i))^2+\tau^2\|g\|_{\Hc_k}^2,
\end{eqnarray}
where $\Hc_k$ and $\|.\|_{\Hc_k}$ denote, respectively, the RKHS corresponding to a kernel $k$ and its norm. 

In the GP approach, $\mu_n$ is interpreted as the posterior mean of a centered surrogate
GP model with kernel $k$, conditioned on $\Dc_n$, assuming zero-mean Gaussian observation noise with variance $\tau^2$.  The KRR and GP interpretations are known to be equivalent; e.g., see,~\cite{Kanagawa2018}. 


Under the GP approach, the posterior covariance of the (surrogate) GP model is often leveraged as a measure of uncertainty in the regression model. The posterior covariance is given in closed form as
\begin{multline}
k_{\tau, X_n}(\cdot,\cdot') = k(\cdot,\cdot') \\ - k^{\top}_{X_n}(\cdot)\left(k_{X_n,X_n}+\tau^2 I_n\right)^{-1}k_{X_n}(\cdot'). \label{eq:k_posterior}
\end{multline}

We define the posterior variance as $\sigma^2_{k,\tau,X_n}(\cdot) = k_{\tau,X_n}(\cdot,\cdot)$. We use the simplified notations of $k_n=k_{\tau,X_n}$ and $\sigma_n= \sigma_{k,\tau,X_n}$, when the values of the omitted variables are clear from the context. 

\subsection{Sparse Approximation Methods}\label{sec:appxre}

Kernel-based models provide powerful prediction and uncertainty estimation (posterior covariance) tools for very general nonlinear
regression. 
However, the computation quickly becomes prohibitive due to the high $\Oc(n^3)$ cost of matrix inversion.  Sparse approximation methods are thus essential tools for applications of kernel-based models to large-scale learning problems. 
A relatively small number $m \ll n$ of \emph{inducing points} can be used to obtain approximate prediction and uncertainty estimates with an efficient computation time of $\Oc(nm^2)$.  In the following, we give an overview of the relevant sparse approximation methods. 

Let $Z_m=[z_1,z_2,...,z_m]^\top\in \Xc^{m}$ be the set of inducing points. The approximate regressor and posterior covariance of the (surrogate) GP model are given as~\citep{wild2021connections}:
\begin{eqnarray}\nn
\mubar_{k, \tau,Z_m,X_n,Y_n}(\cdot) &=& V_n^{\top}(\cdot)Y_n, \label{eq:mubar} \\\nn
\kbar_{\tau,Z_m,X_n}(\cdot,\cdot') &=& k(\cdot,\cdot') - k^{\top}_{Z_m}(\cdot)k^{-1}_{Z_m,Z_m}k_{Z_m}(\cdot')\\\label{eq:appxstatistics}
&&\hspace{-10em}+ k^{\top}_{Z_m}(\cdot)\left(k_{Z_m,Z_m}+\frac{1}{\tau^2}k^{\top}_{X_n,Z_m}k_{X_n,Z_m}\right)^{-1}k_{Z_m}(\cdot'),~~
\end{eqnarray}
where we define
\begin{eqnarray}\nn
     V^{\top}_n(\cdot) &=& k^{\top}_{Z_m}(\cdot)\left(\tau^2k_{Z_m,Z_m}+k^{\top}_{X_n,Z_m}k_{X_n,Z_m}\right)^{-1}\\
    &&\hspace{10em}\times k_{Z_m,X_n}, \label{eq:defV}
\end{eqnarray} 
and where $k_{Z_m}(\cdot)=[k(\cdot,z_1), k(\cdot,z_2),\dots,k(\cdot,z_m)]^{\top}$, $k_{Z_m,Z_m}=[k(z_i,z_j)]_{i,j=1}^m$ is the $m\times m$ kernel matrix over the inducing points, and $k_{X_n,Z_m} = {[k(x_i,z_j)]_{\substack{i=1,\dots,n\\ j=1,\dots, m}}}$. 
The approximate posterior variance is defined as $\sigmabar^2_{k,\tau,Z_m,X_n}(\cdot)= \kbar_{\tau,Z_m,X_n}(\cdot,\cdot)$.
We use the simplified notations $\bar{\mu}_n = \mubar_{k, \tau,Z_m,X_n,Y_n}$, $\kbar_n = \kbar_{\tau,Z_m,X_n}$, and $\sigmabar_n =\sigmabar_{k,\tau,Z_m,X_n}$, when the omitted variables are clear from the context.

In the KRR approach, $\bar{\mu}_n$ is interpreted as the regularized least squares estimator within a reduced rank RKHS induced by $Z_m$~\citep[e.g., see,][]{wild2021connections}. In particular, define 
\begin{eqnarray}
\kappa(\cdot, \cdot') = k^{\top}_{Z_m}(\cdot)k^{-1}_{Z_m,Z_m}k_{Z_m}(\cdot'), \label{eq:kappa}
\end{eqnarray}
and let $\Hc_{\kappa}$ be the associated RKHS. We then have the following~\citep[e.g., see][]{wild2021connections}:
\begin{eqnarray}
\bar{\mu}_n=\arg\min_{g\in\Hc_{\kappa}} \sum_{i=1}^n(y_i-g(x_i))^2+\tau^2\|g\|^2_{\Hc_\kappa}.
\end{eqnarray}
In the GP approach, the approximate expressions are obtained through a Bayesian variational method. In particular, a variational family of GP distributions is introduced with mean $k^{\top}_{Z_m}(\cdot)k^{-1}_{Z_m,Z_m}\mu^{(v)}$ and covariance $k(\cdot,\cdot') - k^{\top}_{Z_m}(\cdot)k^{-1}_{Z_m,Z_m}(k_{Z_m,Z_m}-\Sigma^{(v)})k^{-1}_{Z_m,Z_m}k_{Z_m}(\cdot')$, where $\mu^{(v)}$ and
$\Sigma^{(v)}$ (the mean vector and the covariance matrix of a postulated multivariate Gaussian distribution on $f_{Z_m}$) are the variational parameters. Following the standard steps in Bayesian variational methods, the variational parameters are chosen to minimize the KL divergence between the approximate and exact posteriors, or equivalently, maximize the evidence lower bound (ELBO).
~\cite{Titsias2009Variational} showed that this convex optimization problem can be explicitly solved, leading to the optimal variational parameters:
\begin{align}
\mu_*^{(v)} 
    &= k_{Z_m,Z_m}\bigl(\tau^2k_{Z_m,Z_m}+k^{\top}_{X_n,Z_m}k_{X_n,Z_m}\bigr)^{-1}  k_{Z_m,X_n}Y_n \nn \\
\Sigma_*^{(v)}
    &= k_{Z_m,Z_m}\bigl(k_{Z_m,Z_m}+\frac{1}{\tau^2}k^{\top}_{X_n,Z_m}k_{X_n,Z_m}\bigr)^{-1} k_{Z_m,Z_m}. \nn
\end{align}
%
%
Substituting $\mu_*^{(v)}$ and $\Sigma_*^{(v)}$ in the mean and covariance of the variational distribution results in the expressions for the approximate statistics given in~\eqref{eq:appxstatistics}. 


\paragraph{Other notations.}
For any function $g:\Xc\rightarrow \Rr$,
we write $g_{X_n}=[g(x_1),g(x_2),\dots,g(x_n)]^{\top}$. We also define the noise vector $E_n=[\epsilon_1,\epsilon_2,\dots,\epsilon_n]^{\top}$, and let $k_{\max}=\sup_{x\in\Xc}k(x,x)$, which is assumed to be finite.

For convenience, the most commonly-used notation is briefly summarized in Appendix \ref{sec:notation_list}.

\section{Confidence Intervals for Approximate Kernel-Based Regression}\label{sec:conf_intervals}

Consider the regression problem and its exact solution given in Section~\ref{sec:exactre}. The posterior variance can be leveraged to quantify the uncertainty in the prediction. In particular, various results exist stating that with probability at least $1-\delta$, the prediction function satisfies $|f(x)-\mu_n(x)|\le \beta(\delta)\sigma_n(x)$ (either for fixed $x$ or simultaneously for all $x$), where the confidence interval width multiplier $\beta(\delta)$ depends on the properties of the observation noise and the complexity of $f$~\citep[e.g., see,][]{srinivas2010gaussian,Chowdhury2017bandit,vakili2021optimal}. In this section, we prove similar confidence intervals for the approximate statistics given in Section~\ref{sec:appxre}. As an intermediate step, we first provide a novel interpretation of the approximate posterior variance, which may be of independent interest.

\subsection{Approximate Posterior Variance} \label{sec:approx_var}

In this section, we prove a novel expression for the approximate posterior variance as the sum of three terms: \emph{i)} the error due to the projection of $f\in\Hc_k$ onto $\Hc_{\kappa}$, \emph{ii)} the prediction error from noise-free observations within $\Hc_{\kappa}$, and \emph{iii)} a term capturing the effect of the noise.
Given a function $f\in\Hc_k$, let
\begin{equation}
     \Pi_{\kappa}[f]=\arg\min_{g\in \Hc_{\kappa}}\|f-g\|_{\Hc_k} \label{eq:ftilde}
\end{equation}
denote the projection of $f\in\Hc_k$ onto $\Hc_{\kappa}$. This projection is the minimum norm interpolator of $f$ over $Z_m$ ($\mu_{k, 0,Z_m, f_{Z_m}}(\cdot) = k^{\top}_{Z_m}(\cdot)k^{-1}_{Z_m,Z_m}f_{Z_m}$), as formalized in the following lemma. 

\begin{lemma}[\cite{wild2021connections}]\label{Lemma:projectin_mean}
For any $f\in \Hc_k$, the projection of $f$ onto $\Hc_{\kappa}$ is equivalent to its minimum norm interpolation over $Z_m$:
$
\Pi_{\kappa}[f] = \mu_{k, 0,Z_m, f_{Z_m}}.
$
\end{lemma}

We are now ready to state our first main result on the decomposition of the approximate posterior variance.  Recall that $V_n(\cdot)$ is defined in \eqref{eq:defV}.

\begin{theorem}\label{the:approximatevariance}
Given a dataset $\Dc_n$, inducing points $Z_m$, and $(\mubar_n,\sigmabar_n)$ given in Section~\ref{sec:appxre},
we have
\begin{eqnarray}\nn
&&\hspace{-2em}\sigmabar^2_n(\cdot)= \sup_{g:\|g\|_{\Hc_k}\le 1}\left(g(\cdot)-\Pi_{\kappa}[g](\cdot)\right)^2\\\nn
&&+ \sup_{g:\|g\|_{\Hc_\kappa}\le 1}\left( g(\cdot)-V_n^{\top}(\cdot)g_{X_n} \right)^2+
\tau^2\big\|V_{n}(\cdot)\big\|_{l^2}^2.
\end{eqnarray}


\end{theorem}

The first term captures the maximum error between any $g$ in the unit ball of $\Hc_k$ and its projection onto $\Hc_{\kappa}$. The second term captures the maximum prediction error from noise-free observations within the unit ball of $\Hc_{\kappa}$. Lastly, the third term measures the impact of the observation noise ($l^2$ norm of the coefficients vector), and helps with bounding the regression error due to observation noise, with high probability.
The proof is provided in Appendix~\ref{app:the:approximatevariance}.

\subsection{Confidence Intervals for Kernel-Based Sparse Approximate Regression}

We use the representation of the approximate posterior variance given in Theorem~\ref{the:approximatevariance} to derive novel confidence intervals for approximate regression methods. We first formalize our assumptions. 

\begin{assumption}\label{ass:f_norm}
The RKHS norm of the objective function $f$ is bounded, i.e., $\|f\|_{\Hc_k}\le C_k$, for some $C_k\in \Rr_{>0}$. 
\end{assumption}

\begin{assumption}\label{ass:noise}
The observation noise terms $\epsilon_i$ are independent sub-Gaussian random variables.  That is, for all $i\in \Nn$, for all $\eta \in \Rr$, and for some $R>0$, the moment generating function of $\epsilon_i$ satisfies $\E[\exp(\eta \epsilon_i)]\le \exp(\frac{\eta^2R^2}{2})$.
\end{assumption}

Assumption~\ref{ass:noise} can be relaxed from sub-Gaussian noise to light-tailed noise (only requiring that the moment generating function exists), at the price of slightly increasing the confidence width. This variation is provided in Appendix~\ref{app:light_tailed}.

The representation of the sparse approximate posterior variance given in Theorem~\ref{the:approximatevariance} holds for any set of inducing points. In order for $\mubar_n$ to be a good prediction of $f$, however, the set of inducing points should be selected appropriately. Interestingly, we will see that the requirements for the set of inducing points can be concisely captured by the spectral norm of the approximation error in the kernel matrix. Let
$\lambda_{\max}$ denote the maximum eigenvalue of $k_{X_n,X_n}-\kappa_{X_n,X_n}$. This quantity appears in our confidence intervals, and turns out to be important for our analysis in the subsequent sections on regression and optimization. A key finding from our results is the following:
\emph{A good method for choosing inducing points is one that ensures an upper bound on~$\lambda_{\max}$. }


We now present our confidence intervals based on sparse approximate statistics. The proof is given in Appendix~\ref{app:the:conf_subG}.

\begin{theorem}\label{the:conf_subG}
Recall the definitions of $\mubar_n$ and $\sigmabar_n$ given in Section~\ref{sec:appxre}. Assume that $X_n$ and $Z_m$ are chosen independent of the observation noise $E_n$. 
Under Assumptions~\ref{ass:f_norm} and~\ref{ass:noise}, the following equations each hold with probability at least $1-\delta$ for any fixed $x \in \Xc$:
\[f(x) \le \mubar_n(x)+\beta(\delta)\sigmabar_n(x),~~~f(x)  \ge \mubar_n(x)-\beta(\delta)\sigmabar_n(x),\;\]
where 
$\beta(\delta) = \left(\Bigl(2+\frac{\sqrt{\lambda_{\max}}}{\tau}\Bigr)C_k+\frac{R}{\tau}\sqrt{2\log\Bigl(\frac{1}{\delta}\Bigr)}\right),\;$
and where $C_k$ and $R$ are the parameters specified in Assumptions~\ref{ass:f_norm} and~\ref{ass:noise}. 
\end{theorem}




If the domain $\Xc$ is finite, then uniform confidence bounds readily follow from this result via a union bound, and $\frac{\delta}{|\Xc|}$ can be substituted for $\delta$.  For continuous domains, more effort is required, and we use a discretization argument.  First, consider the following continuity assumption.
%
\begin{assumption}\label{ass:disc}
For each $n\in\Nn$, there exists a discretization $\Xx$ of $\Xc$ such that, for any $f\in \Hc_k$ with $\|f\|_{\Hc_k}\le C_k$, we have $f(x) - f([x])\le \frac{1}{n}$, where $[x] = {\arg}{\min}_{ x'\in \Xx}||x'-x||_{l^2}$ is the closest point in $\Xx$ to $x$, and $|\Xx|\le cC_k^dn^{d}$, where $c$ is a constant independent of $n$ and $C_k$.
\end{assumption}
Assumption~\ref{ass:disc} is a mild assumption that holds for typical kernels such as SE and Mat{\'e}rn~with $\nu>1$~\citep{srinivas2010gaussian, Chowdhury2017bandit, vakili2021optimal}.

\begin{corollary}\label{cor:conf_cont}
Under the setting of Theorem~\ref{the:conf_subG}, and under Assumption~\ref{ass:disc}, the following equations each hold uniformly in $x \in \Xc$ with probability at least $1-\delta$:
\begin{eqnarray}\nn
f(x) \le \mubar_n(x)+\frac{2}{n}+\betat_n(\delta)\Big(\sigmabar_n(x)+\frac{2}{\sqrt{n}}\Big),\\\nn
f(x) \ge \mubar_n(x)-\frac{2}{n}-\betat_n(\delta)\Big(\sigmabar_n(x)+\frac{2}{\sqrt{n}}\Big),
\end{eqnarray}
where 
$
\betat_n(\delta) = \beta(\frac{\delta}{2\Gamma_n})
$,
$\Gamma_n=  c  \left(C^{\mubar_{n}}_k(\frac{\delta}{2})\right)^dn^{d}$, and $C^{\mubar_n}_k(\delta)=C_k(1+\frac{\sqrt{n}k_{\max}}{\tau})+\frac{\sqrt{n}R}{\tau}\sqrt{2\log(\frac{2n}{\delta})}$; $C_k$, $R$ and $c$ are the constants specified in Assumptions~\ref{ass:f_norm},~\ref{ass:noise}, and~\ref{ass:disc}, respectively, and $k_{\max}=\sup_{x\in\Xc}k(x,x)$. 
\end{corollary}

The proof is given in Appendix~\ref{app:cor:conf_cont}, and involves using Assumption~\ref{ass:disc} to characterize the effect of discretizing on $f$, $\mubar_n$, and $\sigmabar_n$. For this purpose, we derive a high probability upper bound $C^{\mubar_n}_k(\delta)$ on $\|\mubar_n(\cdot)\|_{\Hc_k}$ (see Lemma~\ref{lemma:norm_mu_n}), which appears in the expression of $\Gamma_n$.  A separate argument is also used to characterize the Lipschitz behavior of $\sigmabar_n(\cdot)$. We then set $\Gamma_n$ to be the size of discretization required to ensure $\Oc\big(\frac{1}{\sqrt{n}}\big)$ discretization errors; note that $\Gamma_n$ depends on $n$ according to $\Oc(C_k^dn^{3d/2})$. Effectively, the uniform confidence intervals scale as 
$\betat_n(\delta)=\Oc\big((1+\frac{\sqrt{\lambda_{\max}}}{\tau})C_k+\frac{R}{\tau}\sqrt{d\log(\frac{nC_k}{\delta})}\big)$.

In comparison with the exact models, Theorem \ref{the:conf_subG} is analogous to Theorem~$1$ of~\cite{vakili2021optimal} which proves $1-\delta$ confidence bounds of the form $\mu_n(x)\pm\beta(\delta)\sigma_n(x)$ where $\beta(\delta)=C_k+\sqrt{2\log(\frac{1}{\delta})}$. The confidence interval width multipliers are similar except for the scaling of $C_k$ with a $(2+\frac{\sqrt{\lambda_{\max}}}{\tau})$ factor. We now investigate $\lambda_{\max}$.

\subsection{Inducing Points and $\lambda_{\max}$}

As discussed above, our results show that the effect of the choice of inducing points is concisely captured by $\lambda_{\max}$. In the literature, there exist several methods of selecting inducing points that ensure bounds on $\lambda_{\max}$.  As discussed in the introduction, these methods typically choose the inducing points based on the ridge leverage score (RLS). 
Here we highlight a theoretical guarantee for recursive RLS sampling
from~\cite{musco2016recursive}.  We note, however, that our results apply to any rule for selecting inducing points with a guaranteed bound on $\lambda_{\max}$. 

The result depends on a quantity denoted by $\tilde{d}_{k}(n)$, which is the so called \emph{effective dimension} of the kernel.  The rough idea is that while the feature space for typical kernels is infinite-dimensional, for a finite dataset, the number of features with a significant impact on the model is finite. The effective dimension for a given kernel and dataset is often defined as~\citep{zhang2005learning, Valko2013kernelbandit} 
\begin{eqnarray}
\tilde{d}_k(n)  = \tr\left(K_{X_n,X_n}(K_{X_n,X_n}+\tau^2I_n)^{-1}\right).
\end{eqnarray}

\begin{lemma}[\cite{musco2016recursive}, Theorem~$7$]\label{Lemma:recRLS}
Given $\delta\in(0,1/32)$, there exists a method for choosing the inducing points based on a recursive RLS~\citep[][Algorithm~$2$]{musco2016recursive}, with $\Oc(nm^2)$ time complexity, guaranteeing that with probability $1-\delta$, it holds that $m\le 384\tilde{d}_{k}(n)\log(3\tilde{d}_{k}(n)/\delta)$ and  $\lambda_{\max}\le\tau^2$. 
\end{lemma}

Note that when $\lambda_{\max}\le \tau^2$, the term $(2+\frac{\sqrt{\lambda_{\max}}}{\tau})C_k$ of $\beta(\delta)$ in Theorem \ref{the:conf_subG} simply becomes $3C_k$.

It is also useful to define a related notion of information gain.  To do so, let $F$ be a centered GP on the domain $\Xc$ with kernel $k$. Information gain then refers to the mutual information $\Ic(Y_n; F)$ between the data values $Y_n=[y_i]_{i=1}^n$ and $F$. From the closed form expression of mutual information between two multivariate Gaussian distributions~\citep{cover1999elementsold}, it follows that
$
\Ic(Y_n; F)  = \frac{1}{2}\log\det\left(I_n+\frac{1}{\tau^2}K_{X_n,X_n} \right). \label{eq:mutual_info}
$
We then define the data-independent and kernel-specific \emph{maximal information gain} as follows \cite{srinivas2010gaussian}:
\begin{eqnarray}
\gamma_k(n) = \sup_{X_n\subset\Xc}\Ic(Y_n; F). \label{eq:gamma}
\end{eqnarray}
It is known that the information gain and the effective dimension are the same up to logarithmic factors. Specifically,
we have $\tilde{d}_{k}(n)\le \Ic(Y_n;F)$, and $\Ic(Y_n; F)=\Oc(\tilde{d}_{k}(n)\log(n))$~\citep{Calandriello2019Adaptive}.

We will state our results on the number of inducing points, regression error and optimization regret in terms of $\gamma_k$. For specific kernels, explicit bounds on $\gamma_k$ are given in~\cite{srinivas2010gaussian, vakili2020information}.

\section{Sparse Approximate Kernel-Based Regression with Uncertainty Reduction}\label{sec:regression}

In this section, we study the application of our confidence intervals with sparse approximate statistics to kernel-based regression. We are interested in the uniform concentration of $\bar{\mu}_n$ around the true data generating function $f$ in the form of explicit bounds on $||f-\bar{\mu}_n||_{L^{\infty}}$, in terms of $n$ and $m$. We show that, with $m=\Oct(\gamma_k(n))$ and suitably-chosen data points,  $||f-\bar{\mu}_n||_{L^{\infty}}$ converges to zero with high probability. In addition, we characterize the precise convergence rate.
For a more clear presentation of our analysis, we first consider the case of a finite domain $|\Xc|<\infty$, and then show the extension of the results to compact subsets of $\Rr^d$. 
Next, we describe the collection of a dataset based on uncertainty reduction, and then state our result on the error rate of the regression problem.
\subsection{Dataset}\label{sec:dataset}

We recursively collect a dataset based on a maximal uncertainty reduction design. At each discrete time $j=1,2,\dots$, a point $x_j$ is chosen based on the rule $x_j= \arg\max_{x\in\Xc}\sigmabar_{j-1}(x)$, where $\sigmabar^2_{j-1}=\sigmabar^2_{k,\tau,Z^{(j-1)}_{m_{j-1}},X_{j-1}}$ is the approximate posterior variance based on previous observation points $X_{j-1}$, using a set of inducing points $Z^{(j-1)}_{m_{j-1}}$. 

For correctness, we assume that the set of inducing points at each time $j$ is selected based on the recursive RLS method~\citep[][Algorithm~$2$]{musco2016recursive} formalized in Lemma~\ref{Lemma:recRLS}. 
Let $\delta_j$ denote the confidence parameter of Lemma~\ref{Lemma:recRLS} in selecting $Z^{(j)}_{m_j}$.
Then, with probability at least $1-\delta_j$, $m_j=\Oc(\tilde{d}_k(j)\log(\tilde{d}_k(j))/\delta_j)$ and $\lambda^{(j)}_{\max}\le \tau^2$, where $\lambda^{(j)}_{\max}$ denotes the maximum eigenvalue of $k_{X_j,X_j}-\kappa_{X_j,X_j}$. 
The set of inducing points may be selected based on any other method guaranteeing similar bounds on $\lambda^{(j)}_{\max}$. 

Intuitively, this dataset $\Dc_n=\{(x_j,y_j)_{j=1}^n\}$ is roughly evenly spread over the entire domain, which allows us to prove the convergence of $\mubar_n$ to $f$. Note that without assumptions on the dataset, we cannot expect such a result. For example, if all data points are collected from a small region in the domain, it is not expected for the prediction error to be small at the points that are far from this region.

\subsection{Uniform Convergence of $\mubar_n$ to $f$}

We prove that for the dataset described above, $\mubar_n$ uniformly converges to $f$ with high probability.

\begin{theorem}\label{the:regression}
For $\Dc_n$ described above with $\delta_j=\frac{3\delta}{\pi^2j^2}$, under Assumptions~\ref{ass:f_norm} and~\ref{ass:noise}, we have with probability at least  $1-\delta$ that
\begin{eqnarray}\nn
\|f-\bar{\mu}_n\|_{L^{\infty}} 
= \Oc\left(\sqrt{\frac{\gamma_k(n)}{n}\log\Big(\frac{|\Xc|}{\delta}\Big)}\right).
\end{eqnarray}
\end{theorem}


We emphasize that by Lemma~\ref{Lemma:recRLS}, the number of inducing points used in computation of $\mubar_n$ is bounded as $m_{n}=\Oct(\gamma_k(n)\log(n^2\gamma_k(n)/\delta))$.

The proof of Theorem~\ref{the:regression} relies on the confidence intervals proven in Theorem~\ref{the:conf_subG}, and a bound on the ratio $\sigmabar^2_n(x)/\sigma^2_n(x)$ between the approximate and exact posterior variances of the surrogate GP model, which we show to be bounded by a constant depending on $\lambda_{\max}$ and $\tau$. The proof, and a detailed expression of the bound including the hidden terms in $\Oc$ notation, are provided in Appendix~\ref{app:the:regression}.



\subsection{Extension to Continuous Domains}\label{sec:cont_reg}

The following theorem extends the above convergence result to compact domains $\Xc\subset \Rr^d$. 

\begin{theorem}\label{the:regressioncont}
Consider $\Dc_n$ described in Section~\ref{sec:dataset} with $\delta_j=\frac{3\delta}{\pi^2j^2}$. Under Assumptions~\ref{ass:f_norm},~\ref{ass:noise} and~\ref{ass:disc}, we have with probability at least  $1-\delta$ that
\begin{eqnarray}
\|f-\bar{\mu}_n\|_{L^{\infty}}  = \Oc\left(\sqrt{\frac{d\gamma_k(n)}{n}\log\Big(\frac{n}{\delta}\Big)}\right).
\end{eqnarray}
\end{theorem}
The proof of Theorem~\ref{the:regressioncont} follows the same steps as the proof of Theorem~\ref{the:conf_subG}, using the uniform confidence intervals from Corollary~\ref{cor:conf_cont} instead of the confidence intervals  from Theorem~\ref{the:conf_subG}. 
The proof and a detailed expression of the bound including the hidden terms in $\Oc$ notation is provided in Appendix~\ref{app:the:regression}. 


\begin{remark}
Our results can be specified for the Mat{\'e}rn and SE kernels using the bounds on $\gamma_{k}(n)$ from \cite{srinivas2010gaussian} and \cite{vakili2020information}. For the SE kernel, we have $\gamma_{k}(n) =\Oc((\log(n))^{d+1})$; thus, setting $m=\Oc((\log(n))^{d+2})$, we have with probability at least $1-\delta$ that $\|f-\mubar_n\|_{L^{\infty}}=\Oct\big(\sqrt{\frac{d}{n}\log(\frac{n}{\delta})}\big).\;$
In the case of Mat{\'e}rn kernel, we have $\gamma_{k}(n) =\Oct(n^{\frac{d}{2\nu+d}})$; thus, setting $m=\Oct(n^{\frac{d}{2\nu+d}})$, we have with probability at least $1-\delta$ that
$
\|f-\mubar_n\|_{L^{\infty}}=\Oct\big(n^{\frac{-\nu}{2\nu+d}}\sqrt{d\log(\frac{n}{\delta})}\big)
.\;$
\end{remark}

\section{Kernel-Based Black-Box Optimization Using Sparse Approximation Methods}\label{sec:optimization}

In this section, we consider the problem of optimizing a black-box function $f$ over $\Xc$. In this setting, a learning algorithm sequentially selects observation points $x_n$ indexed by a discrete time $n=1,2,\dots$, and receives the corresponding real-valued noisy observations $f(x_n)+\epsilon_n$, where $\epsilon_n$ is the observation noise. The goal is to minimize the (cumulative) regret, defined as the sum of the losses compared to the maximum attainable objective, over a time horizon $N$:
\begin{eqnarray}
\Rc(N) = \sum_{n=1}^N(f(x^*)-f(x_n)),
\end{eqnarray}
where $x^*\in \arg\max_{x\in\Xc}f(x)$ is a global maximum of $f$. 


\subsection{Algorithm}

Our S-BPE algorithm is an adaptation of BPE algorithm proposed in~\citep{li2021gaussian}, where the exact GP statistics are replaced by their sparse approximations. The algorithm proceeds in batches. Specifically, the time horizon $N$ is partitioned into $B$ batches indexed by $i=1,\dots, B$, and the length of the $i$-th batch is denoted by $N_i$.

Similar to the previous section, we first consider the case that $|\Xc|<\infty$, and then show the extension of the results to compact subsets of $\Rr^d$. 
The S-BPE algorithm maintains a set $\Xc_i$ of the potential maximizers of $f$ (initialized to be the full domain $\Xc$). During each batch $i$ the algorithm explores the points in $\Xc_i$, and at the end, reduces the size of $\Xc_i$ by eliminating some of its elements. 

For exploration, the algorithm picks the points in $\Xc_i$ with the highest uncertainty, which is measured by the sparse approximate posterior variance. To ensure the validity of the confidence intervals, we must ignore the previous batches and compute the variance only based on the observations in the current batch. 
We thus introduce the notations
\begin{eqnarray}\nn
\bar{\mu}_{j-1,i}(\cdot)&=& \mubar_{k,\tau, Z^{(j-1,i)}_{m_{j-1,i}},X_{j-1,i} Y_{j-1,i}}(\cdot),\\\nn
\bar{\sigma}^2_{j-1,i}(\cdot) &=& \bar{\sigma}^2_{k,\tau,Z^{(j-1,i)}_{m_{j-1,i}},X_{j-1,i}}(\cdot),
\end{eqnarray}
for the approximate posterior mean and variance based on the first $j-1$ observations $\{X_{j-1,i},Y_{j-1,i}\}$ in batch $i$. 

During each batch $i$, the set of inducing points at each time $j$ is selected based on the recursive RLS method~\citep[][Algorithm~$2$]{musco2016recursive} formalized in Lemma~\ref{Lemma:recRLS}. 
Let $1-\delta_j$ denote the confidence level of Lemma~\ref{Lemma:recRLS} in selecting $Z^{(j,i)}_{m_{j,i}}$.
Then, with probability at least $1-\delta_j$, $m_{j,i}=\Oc(\tilde{d}_k(j)\log(\tilde{d}_k(j))/\delta_j)$ and $\lambda^{(j,i)}_{\max}\le \tau^2$, where $\lambda^{(j,i)}_{\max}$ denotes the maximum eigenvalue of $k_{X_{j,i},X_{j,i}}-\kappa_{X_{j,i},X_{j,i}}$.



The $j$-th point selected in batch $i$ is specified as $x_j=\arg\max_{x\in \Xc_i}\sigmabar_{j-1,i}(x)$. 

\begin{remark}
The selection of points in each batch $i$ (as well as the design of the dataset in Section~\ref{sec:regression})
is based on uncertainty reduction, and only uses the sparse approximate variance. This ensures the validity of the confidence interval in Theorem~\ref{the:conf_subG}, which requires the observation points and the inducing points to be selected independently from the observation noise (i.e., without using $Y_{j-1,i}$).
\end{remark}

At the end of batch $i$, the elements of $\Xc_i$ that are unlikely to be the maximizers of $f$ are removed. To do this, Theorem~\ref{the:conf_subG} is used to create the following upper and lower confidence bounds:
\begin{eqnarray}\nn
u_i(\cdot)&=&\mubar_{N_i, i}(\cdot) + \beta\Big(\frac{\delta}{4B|\Xc|}\Big)\sigmabar_{N_i,i}(\cdot), \\\label{eq:conf_batch}
l_i(\cdot)&=&\mubar_{N_i, i}(\cdot) - \beta\Big(\frac{\delta}{4B|\Xc|}\Big)\sigmabar_{N_i,i}(\cdot) .
\end{eqnarray}
Using $l_i$ and $u_i$, S-BPE sets
$\Xc_{i+1}=\{x\in\Xc_i: u_i(x)\ge \max_{x'\in \Xc_i} l_i(x')\}$.
The rationale here is that conditioned on the validity of the confidence bounds, if $u_i(x)<l_i(z)$ for some $z\in \Xc_i$, then $x$ cannot be a maximizer of $f$. Pseudo-code for S-BPE is provided in Algorithm~\ref{Alg:S-PBE} in Appendix~\ref{app:alg_SBPE}.



\begin{theorem}\label{the:optimization}
Consider the black-box optimization problem described above. In S-BPE, set $\delta_j=\frac{3\delta }{\pi^2Bj^2}$, and $l_i$, $u_i$ according to~\eqref{eq:conf_batch}. Under Assumptions~\ref{ass:f_norm} and~\ref{ass:noise}, the regret performance of S-BPE satisfies the following with probability at least $1-\delta$:
\begin{eqnarray}
\Rc(N) = \Oct\left(\sqrt{N\gamma_k(N)\log\Big(\frac{|\Xc|}{\delta}\Big)}\right).
\end{eqnarray}
\end{theorem}

The proof is given in Appendix~\ref{app:the:optimization}, along with a more precise expression for the bound showing constants and logarithmic terms.

\subsection{Extension to Continuous Domains}

The extension to the continuous domains is similar to Section~\ref{sec:cont_reg}. 
For continuous domains, we replace the confidence bounds $l_i$ and $u_i$ in~\eqref{eq:conf_batch} with the following:
\begin{eqnarray}\nn
&&\hspace{-2.2em}u_i(\cdot)=\mubar_{N_i,i}(\cdot)+\frac{2}{N_i}+\betat_{N_i}\big(\frac{\delta}{4B}\big)\Big(\sigmabar_{N_i,i}(\cdot)+\frac{2}{\sqrt{N_i}}\Big), \\\label{eq:conf_cont}
&&\hspace{-2.2em}l_i(\cdot)=\mubar_{N_i,i}(\cdot)-\frac{2}{N_i}-\betat_{N_i}\big(\frac{\delta}{4B}\big)\Big(\sigmabar_{N_i,i}(\cdot)+\frac{2}{\sqrt{N_i}}\Big),~~~~~~~~
\end{eqnarray}
which are uniformly valid as per Corollary~\ref{cor:conf_cont}. 

\begin{theorem}\label{the:optimizationcont}
Consider the black-box optimization problem described above. 
In S-BPE, set $\delta_j=\frac{3\delta }{\pi^2Bj^2}$, and $l_i$, $u_i$ according to~\eqref{eq:conf_cont}. Under Assumptions~\ref{ass:f_norm},~\ref{ass:noise} and~\ref{ass:disc}, the regret of S-BPE satisfies the following with probability at least $1-\delta$:
\begin{eqnarray}\nn
\Rc(N)
&=&\Oct\left(
\sqrt{Nd\gamma_k(N)\log\Big(\frac{N}{\delta}\Big)}
\right).
\end{eqnarray}
\end{theorem}
The proof and a more precise expression for the bound are given in Appendix~\ref{app:the:optimization}.

Using Lemma~\ref{Lemma:recRLS}, the total time complexity of S-BPE over the time horizon $N$ is bounded by $\Oct(N^2(\gamma_k(N))^2)$, which is the same as that of BKB~\citep{Calandriello2019Adaptive} and lower than that of S-GP-TS~\citep{Vakili2020Scalable}. All are significant improvements over the time complexity of standard kernel-based bandit algorithms using exact GP statistics.


We note that the time complexity $\Oct(N^2(\gamma_k(N))^2)$ excludes the complexity of optimizing the acquisition function (here, sparse approximate variance).  This aspect is largely orthogonal to our focus on the approximate posterior calculation, and it is similarly not considered in the most relevant works on optimization using sparse GPs~\cite{Calandriello2019Adaptive, Vakili2020Scalable}. 

\begin{remark}

%
In the case of the SE kernel, Theorem~\ref{the:optimization} implies the following regret bound:
\begin{eqnarray}
\Rc(N)=\Oct\left(\sqrt{Nd(\log(N))^{d+1}\log\Big(\frac{N}{\delta}\Big)}\right),
\end{eqnarray}
and in the case of Mat{\'e}rn kernels, Theorem~\ref{the:optimization} implies
\begin{eqnarray}
\Rc(N)=\Oct\left(N^{\frac{\nu+d}{2\nu+d}}(\log(N))^{\frac{\nu}{2\nu+d}}\sqrt{d\log\Big(\frac{N}{\delta}\Big)}\right).
\end{eqnarray}
Comparing to the lower bounds proven in~\cite{Scarlett2017Lower}, both of these bounds are tight up to logarithmic factors in $N$. 
To the best of our knowledge, this is the first kernel-based bandit optimization algorithm based on sparse approximation methods enjoying such near-optimality guarantees.
In comparison, for Mat{\'e}rn kernels, the regret bounds for both BKB and S-GP-TS scale at a rate $\Oct(N^{\frac{2\nu+3d}{4\nu+2d}})$, which become super-linear (trivial) when $\nu\le d/2$.  

\end{remark}

\section{Conclusion}

We have proved confidence intervals for kernel-based regression with sparse approximate models, using a novel representation of the sparse approximate posterior variance. 
We demonstrated the applications of the confidence intervals for regression problems, as well as kernel-based bandit optimization problems.  For the latter, our regret bounds are optimal up to logarithmic factors for the cases where lower bounds on the regret are known.


\textbf{Acknowledgment.} J.~Scarlett was supported by the Singapore National Research Foundation (NRF) under grant number R-252-000-A74-281.

\bibliography{references.bib}
\bibliographystyle{icml2022}

\newpage
\appendix
\onecolumn



In the following appendices, we provide a summary of some common notation, pseudo-code for S-BPE, the proofs of our theorems, as well as a variation of Theorem~\ref{the:conf_subG} relaxing the sub-Gaussian noise assumption. The proofs of auxiliary lemmas are provided in Appendix~\ref{app:proof_lemmas}. 

\section{Summary of Main Notation} \label{sec:notation_list}

For the reader's convenience, some of the main notation used throughout the paper is summarized as follows:
\begin{itemize}
    \item Inputs $X_n= [x_1,...,x_n]^{\top}$, observations $Y_n= [y_1,...,y_n]^{\top}$, noise terms $E_n=[\epsilon_1,\epsilon_2,\dots,\epsilon_n]^{\top}$, inducing points $Z_m=[z_1,z_2,...,z_m]^\top$, regularization parameter $\tau$.
    \item Exact kernel $k(\cdot,\cdot)$, reduced rank kernel $\kappa(\cdot,\cdot)$ (see \eqref{eq:kappa}).
    \item Exact posterior mean (or kernel regressor) $\mu_{k,\tau, X_n, Y_n}(\cdot)$, covariance $k_{\tau, X_n}(\cdot,\cdot)$, and variance $\sigma^2_{k,\tau,X_n}(\cdot)$ (see \eqref{eq:regressor} and \eqref{eq:k_posterior}), or in shorthand notation, $\mu_n$, $k_n$, and $\sigma_n$.
    \item Approximate posterior mean $\mubar_{k, \tau,Z_m,X_n,Y_n}(\cdot)$, covariance $\kbar_{\tau,Z_m,X_n}(\cdot,\cdot')$, and variance $\sigmabar^2_{k,\tau,Z_m,X_n}(\cdot)$ (see \eqref{eq:appxstatistics}), or in shorthand notation, $\bar{\mu}_n$, $\bar{k}_n$, and $\bar{\sigma}_n$.
    \item Auxiliary vector-valued function $V_n$ (see \eqref{eq:defV}) such that $\mubar_n(\cdot) = V_n^{\top}(\cdot)Y_n$.
    \item (Section \ref{sec:conf_intervals}) Projected function $\Pi_{\kappa}[f]$ (see \eqref{eq:ftilde}), RKHS norm bound $C_k$, sub-Gaussianity parameter $R$, maximum eigenvalue $\lambda_{\max}$ (of $k_{X_n,X_n} - \kappa_{X_n,X_n}$), confidence interval multiplier $\beta(\delta)$, discretized input $[x]$.
\end{itemize}


\section{Pseudo-Code for S-BPE}\label{app:alg_SBPE}

Pseudo-code for S-BPE is provided in Algorithm~\ref{Alg:S-PBE} below. 

\begin{algorithm}
\caption{Sparse Batched Pure Exploration (S-BPE)}\label{Alg:S-PBE}
\begin{algorithmic}[1]

\STATE \textbf{Initialization:} $k$, $\Xc$, $f$, $N$, $N_0\leftarrow 1$ 
\FOR{$i\leftarrow 1,2,\dots$}
\STATE $N_i\leftarrow \lceil \sqrt{N \cdot N_{i-1}}\rceil$ \label{line:Ni}
\STATE $\Sc_i\gets\emptyset$
\FOR{$j\leftarrow 1,2,\dots,N_i$}
\STATE Update $Z^{(j-1,i)}_{m_{j-1},i}$ 
\STATE Update $\bar{\sigma}_{j-1,i}(\cdot)$ 
\STATE $x_n\leftarrow \arg\max_{x\in\Xc_i}\sigmabar_{j-1,i}(x)$
\STATE $\Sc_i\gets\Sc_i\cup\{x_n\}$
\STATE $n\leftarrow n+1$
\IF{$n>N$}
\STATE Terminate
\ENDIF
\STATE Collect observations for all points in $\Sc_i$
\ENDFOR
\STATE Compute $\bar{\mu}_{N_i,i}$ and $\bar{\sigma}_{N_i,i}$ 
\STATE $\Xc_{i+1}\leftarrow\{x\in\Xc_i: u_i(x)\ge \max_{x'\in\Xc_i} l_i(x')\}$\label{line:elim}
\ENDFOR\label{euclidendwhile}
\end{algorithmic}
\end{algorithm}

For discrete and continuous domains $l_i$ and $u_i$ are set according to~\eqref{eq:conf_batch} and~\eqref{eq:conf_cont}, respectively.


\section{Proof of Theorem~\ref{the:approximatevariance} (Decomposition of the Approximate Posterior Variance)} \label{app:the:approximatevariance}

From the expression for the approximate posterior covariance in \eqref{eq:appxstatistics}, we have
\begin{eqnarray}\nn
\sigmabar^2_n(\cdot,\cdot) &=& \underbrace{k(\cdot,\cdot) - k^{\top}_{Z_m}(\cdot)k^{-1}_{Z_m,Z_m}k_{Z_m}(\cdot)}_{\text{Term}~1}+ \underbrace{k^{\top}_{Z_m}(\cdot)\left(k_{Z_m,Z_m}+\frac{1}{\tau^2}k^{\top}_{X_n,Z_m}k_{X_n,Z_m}\right)^{-1}k_{Z_m}(\cdot)}_{\text{Term}~2}. 
\end{eqnarray}

Term $1$ is equivalent to the posterior variance of a noise-free GP model with kernel $k$ conditioned on $Z_m$:
\begin{eqnarray}
\text{Term}~1 = \sigma^2_{k,0,Z_m}(\cdot).
\end{eqnarray}

The following lemma shows that Term $2$ is equivalent to the posterior variance of a GP model with kernel $\kappa$ conditioned on~$X_n$.  

\begin{lemma}\label{lemma:Term2}
For any kernel $k$ and any $Z_m$, $X_n$, and $Y_n$, we have
\begin{eqnarray}\nn
\mathrm{Term}~2 = \sigma^2_{\kappa,\tau, X_n}(\cdot),
\end{eqnarray}
where $\kappa(\cdot,\cdot')=k_{Z_m}(\cdot)k^{-1}_{Z_m,Z_m}k_{Z_m}(\cdot')$.
\end{lemma}

In addition, we have the following expression for $V_n(\cdot)$ based on $\kappa$.
\begin{lemma}\label{lemma:V_nkappa}
For any kernel $k$ and any $Z_m$, $X_n$, and $Y_n$, we have
\begin{eqnarray}
V^{\top}_n(\cdot) = \kappa^{\top}_{X_n}(\cdot)(\tau^2I_n+\kappa_{X_n,X_n})^{-1},
\end{eqnarray}
where $\kappa(\cdot,\cdot')=k_{Z_m}(\cdot)k^{-1}_{Z_m,Z_m}k_{Z_m}(\cdot')$.
\end{lemma}

Combining \eqref{eq:regressor} and \eqref{eq:mubar} with Lemma \ref{lemma:V_nkappa}, we find that $\bar{\mu}_{k,\tau,Z_m, X_n,\Pi_{\kappa}[f]_{X_n}} = \mu_{\kappa,\tau, X_n,  \Pi_{\kappa}[f]_{X_n} }$, where $\Pi_{\kappa}[f]_{X_n}=[\Pi_{\kappa}[f](x_1), \dots, \Pi_{\kappa}[f](x_n)]^{\top}$ denotes the vector of values of the projection of $f$ onto $\Hc_\kappa$, at the data points.

Then, Theorem~\ref{the:approximatevariance} follows from the following two lemmas on the posterior variance of the exact GP models.

\begin{lemma}[e.g., see, \cite{Kanagawa2018}, Proposition~$3.10$]\label{lemma:noise-freevar}
For any positive definite kernel $k$ on $\Xc$ and $Z_m\in\Xc^m$, we have
\begin{eqnarray}\nn
\sigma_{k,0,Z_m}(\cdot) = \sup_{g:\|g\|_k \le 1}|g(\cdot)-\mu_{k,0,Z_m,g_{Z_m}}(\cdot)|.
\end{eqnarray}
\end{lemma}

\begin{lemma}[\cite{vakili2021optimal}, Proposition~$1$]\label{lemma:noisyvar}
For any positive definite kernel $k$ on $\Xc$ and $X_n\in\Xc^n$, we have
\begin{eqnarray}\nn
&&\sigma^2_{k,\tau,X_n}(\cdot) = \sup_{g:\|g\|_{\Hc_k}\le 1}(g(\cdot)-\mu_{k,\tau,X_n,g_{X_n}}(\cdot))^2+ \tau^2\left\|k_{X_n}(\cdot)(\tau^2I_n+k_{X_n,X_n})^{-1}\right\|^2_{l^2}.
\end{eqnarray}
\end{lemma}

Combining Lemmas~\ref{lemma:Term2} and ~\ref{lemma:noise-freevar}, we have
\begin{eqnarray}\nn
\text{Term}~1 &=&\sigma^2_{k,0,Z_m}(\cdot)\\\nn
&=&\sup_{g:\|g\|_{\Hc_k}\le 1}\left(g(\cdot)-\mu_{k, 0,Z_m, g_{Z_m}}(\cdot)\right)^2
\\\label{eq:term1}
&=&\sup_{g:\|g\|_{\Hc_k}\le 1}\left(g(\cdot)-\Pi_{\kappa}[g](\cdot)\right)^2,
\end{eqnarray}
where the last line holds by Lemma~\ref{Lemma:projectin_mean}. 

In addition, using Lemmas~\ref{lemma:Term2},~\ref{lemma:noisyvar} and~\ref{lemma:V_nkappa} in that order, we obtain the following:
\begin{eqnarray}\nn
\text{Term}~2 &=& \sigma^2_{\kappa,\tau, X_n}(\cdot)
\\\nn
&=&
\sup_{g:\|g\|_{\Hc_\kappa}\le 1}\left( g(\cdot)-\mu_{\kappa,\tau,X_n,g_{X_n}}(\cdot) \right)^2+
\tau^2\Big\|\kappa_{X_n}(\cdot)(\tau^2I_n+\kappa_{X_n,X_n})^{-1}\Big\|_{l^2}^2
\\\label{eq:term2}
&=&\sup_{g:\|g\|_{\Hc_\kappa}\le 1}\left( g(\cdot)-V_n^{\top}(\cdot)g_{X_n} \right)^2+
\tau^2\Big\|V_{n}(\cdot)\Big\|_{l^2}^2.
\end{eqnarray}

Putting~\eqref{eq:term1} and~\eqref{eq:term2} together, the proof is complete.


\section{Proof of Theorem~\ref{the:conf_subG} (Confidence Interval)}\label{app:the:conf_subG}

Recall the definitions of $V_n^{\top}(\cdot)$ and $\Pi_{\kappa}[f]$ in \eqref{eq:defV} and \eqref{eq:ftilde}. We can decompose the prediction error $f(x)-\mubar_n(x)$ as follows:
\begin{eqnarray}\nn
f(x) - \mubar_n(x) &=& f(x) - V_n^{\top}(x)Y_n\\\nn
&=& f(x) - V_n^{\top}(x)f_{X_n} - V_n^{\top}(x)E_n\\\label{eq:errorexpression}
&=& \underbrace{f(x) -\Pi_{\kappa}[f](x)}_{\text{Term}~1} + \underbrace{\Pi_{\kappa}[f](x)
- V_n^{\top}(x)\Pi_{\kappa}[f]_{X_n}}_{\text{Term}~2}
- \underbrace{V_n^{\top}(x)(f_{X_n}-\Pi_{\kappa}[f]_{X_n})}_{\text{Term}~3}
- \underbrace{V_n^{\top}(x)E_n}_{\text{Term}~4}.
\end{eqnarray}

The first term corresponds to the error due to projection of $f$ onto the approximate RKHS $\Hc_{\kappa}$. The second term corresponds to the prediction error from noise-free observations within $\Hc_{\kappa}$. The third term corresponds to the error due to projection of observations onto $\Hc_{\kappa}$, which can be interpreted as an observation noise. Finally, the fourth term corresponds to the effect of noise. We bound each term based on the expression for the approximate posterior variance given in Theorem~\ref{the:approximatevariance}.

\paragraph{Term~$1$:} The first term is bounded as follows. 
\begin{eqnarray}\nn
|f(x)-\Pi_{\kappa}[f](x)|&\le&\|f\|_{\Hc_k} \sigmabar_{n}(x)\\\nn
&\le& C_k \sigmabar_{n}(x),
\end{eqnarray}
where the first line holds by Theorem~\ref{the:approximatevariance} (after lower bounding the second and third terms therein by zero), and by the fact that rescaling $f$ by its RKHS norm creates a function with RKHS norm exactly one.  By Lemma~\ref{Lemma:projectin_mean}, this rescaling amounts to the same rescaling of $\Pi_{\kappa}[f]$. The second line follows directly from the definition of $C_k$. 

\paragraph{Term~$2$:} The  second term is bounded as follows:
\begin{eqnarray}\nn
|\Pi_{\kappa}[f](x)
- V_n^{\top}(x)\Pi_{\kappa}[f]_{X_n}|
&\le& \|\Pi_{\kappa}[f]\|_{\Hc_{\kappa}}\sigmabar_{n}(x)\\\nn
&\le&C_k\sigmabar_{n}(x),
\end{eqnarray}
where the first line holds by Theorem~\ref{the:approximatevariance} (again, rescaling $\Pi_{\kappa}[f]$ by its RKHS norm creates a function with RKHS norm exactly one, and this rescaling amounts to the same rescaling of $ V_n^{\top}(x)\Pi_{\kappa}[f]_{X_n}$). The second line comes from the definition of $C_k$, and the fact that $\|\Pi_{\kappa}[f]\|_{\Hc_{\kappa}}= \|\Pi_{\kappa}[f]\|_{\Hc_{k}}\le \|f\|_{\Hc_{k}}$, by definition of $\Pi_{\kappa}[f]$.

\paragraph{Term~$3$:} To bound the third term, we use the following known result. 

\begin{lemma}[e.g., see \cite{Kanagawa2018}, Lemma~$3.9$]\label{Lemma:supRKHSball}
For any kernel $k$ on $\Xc, $ $x_1,x_2, \dots, x_n\in\Xc$ and $c_1, c_2, \dots, c_n\in \Rr$, the following holds:
\begin{eqnarray}
\sup_{f: \|f\|_{\Hc_k}\le 1}\sum_{i=1}^n c_if(x_i) =\bigg\|\sum_{i=1}^nc_ik(\cdot,x_i)\bigg\|_{\Hc_k}.
\end{eqnarray}
\end{lemma}

Define $\kappa^{\perp} =k-\kappa$. Using Lemma~\ref{Lemma:supRKHSball}, we can write
\begin{eqnarray}\nn
\sup_{\|g\|_{\Hc_{\kappa^{\perp}}}\le 1} (V_n^{\top}(x)g_{X_n})^2 &=& \|V_n^{\top}(x)\kappa^{\perp}_{X_n}(\cdot)\|_{\Hc_{\kappa^{\perp}}}^2\\\nn
&&\hspace{-6em}= \langle V_n^{\top}(x)\kappa^{\perp}_{X_n}(\cdot), V_n^{\top}(x)\kappa^{\perp}_{X_n}(\cdot) \rangle_{\Hc_{\kappa^{\perp}}}\\\nn
&&\hspace{-6em}=V_n^{\top}(x)\kappa^{\perp}_{X_n,X_n}V_n(x)\\\nn
&&\hspace{-6em}= V_n^{\top}(x)(k_{X_n,X_n}-\kappa_{X_n,X_n})V_n(x)\\\nn
&&\hspace{-6em}\le
\lambda_{\max}\|V_n(x)\|_{l^2}^2\\\label{eq:tildperp}
&&\hspace{-6em}\le\frac{\lambda_{\max}\sigmabar^2_{n}(x)}{\tau^2},
\end{eqnarray}
where the third equality uses the reproducing property, the first inequality comes from the definition of $\lambda_{\max}$, and the second inequality holds by Theorem~\ref{the:approximatevariance}.

With $\Pi_{\kappa}^{\perp}[f]=f-\Pi_{\kappa}[f]$, and rescaling by $\|\Pi_{\kappa}^{\perp}[f]\|_{\Hc_{\kappa^{\perp}}}$, Equation~\eqref{eq:tildperp} implies
\begin{eqnarray}\nn
V_n^{\top}(x)(f_{X_n}-\Pi_{\kappa}[f]_{X_n}) &\le& \frac{\|\Pi_{\kappa}^{\perp}[f]\|_{\Hc_{\kappa^{\perp}}}\sqrt{\lambda_{\max}}\sigmabar_n(x)}{\tau}
\\\nn
&\le&\frac{C_k\sqrt{\lambda_{\max}}\sigmabar_n(x)}{\tau},
\end{eqnarray}
where the second inequality uses $\|\Pi_{\kappa}^{\perp}[f]\|_{\Hc_{\kappa^{\perp}}}=\|\Pi_{\kappa}^{\perp}[f]\|_{\Hc_{k}}\le \|f\|_{\Hc_k}$, i.e., when decomposing $f$ into the orthogonal components $\Pi_{\kappa}[f]$ and $\Pi_{\kappa}^{\perp}[f]$, their norms cannot exceed that of $f$.

\paragraph{Term $4$:} The fourth term can be bounded with high probability using Chernoff-Hoeffding inequality, noticing that $V^{\top}_n(x)E_n$ is a sub-Gaussian random variable with parameter $R\|V_n(x)\|_{l^2}$, which is bounded by $\frac{R}{\tau}\sigmabar_n(x)$ based on Theorem~\ref{the:approximatevariance}. This is formalized in the following lemma. 

\begin{lemma}\label{lemma:chern}
Under Assumption~\ref{ass:noise}, when the vector $V_n(x)$ is independent of $E_n$, the following equations each hold with probability at least $1-\delta$ for any fixed $x \in \Xc$:
\begin{eqnarray}\nn
V^{\top}_n(x)E_n &\le& \frac{R\sigmabar_n(x)}{\tau}\sqrt{2\log\Big(\frac{1}{\delta}\Big)},\\\nn
V^{\top}_n(x)E_n &\ge& - \frac{R\sigmabar_n(x)}{\tau}\sqrt{2\log\Big(\frac{1}{\delta}\Big)}.
\end{eqnarray}
\end{lemma}


Combining the bounds on the four terms completes the proof of Theorem \ref{the:conf_subG}.

\section{Relaxing the Sub-Gaussian Noise Assumption}\label{app:light_tailed}

Assumption~\ref{ass:noise}, required in the analysis of Theorem~\ref{the:conf_subG}, can be relaxed from sub-Gaussian noise to light-tailed noise, at the price of slightly increasing the confidence interval width. Recall the definition of light-tailed distributions:
A random variable $X$ is called light-tailed if its moment generating function $M(h)$ exists; specifically, there exists $h_0>0$ such that for all $|h|\le h_0$, $M(h) <\infty$.

For a zero mean light-tailed random variable $X$, we have~\citep{chareka2006locally} 
\begin{eqnarray}\label{MGFLT}
M(h)&\le& \exp(\xi_0 h^2/2),~ \forall |h|\le h_0,
\xi_0 = \sup\{M^{(2)}(h), |h|\le h_0\},
\end{eqnarray}
where $M^{(2)}(\cdot)$ denotes the second derivative of $M(\cdot)$, and $h_0$ is the parameter specified above. We observe that the upper bound in~\eqref{MGFLT} is the moment generating function of a zero mean Gaussian random variable with variance $\xi_0$. Thus, light-tailed distributions can also be considered locally sub-Gaussian distributions.
We formalize the light-tailed noise assumption as follows. 
\begin{assumption}[Light-Tailed Noise]\label{ass:light_tailed} The noise terms $\epsilon_i$ are zero mean independent random variables over $i$, and there is a constant $\xi_0>0$ such that
$\forall h\le h_0, \forall n\in \Nn,
\E[e^{h\epsilon_n}]\le \exp(\frac{h^2\xi_0}{2})
$.
\end{assumption}

The following is a variation of Theorem~\ref{the:conf_subG} with Assumption~\ref{ass:light_tailed} in place of Assumption~\ref{ass:noise}. 

\begin{theorem}\label{the:conf_lt}
Recall the definitions of $\mubar_n$ and $\sigmabar_n$ given in Section~\ref{sec:appxre}. Assume $X_n$ and $Z_m$ are chosen independent of the observation noise $E_n$. 
Under Assumptions~\ref{ass:f_norm} and~\ref{ass:light_tailed}, the following equations each hold with probability at least $1-\delta$ for any fixed $x \in \Xc$:
\begin{eqnarray}\nn
f(x) \le \mubar_n(x)+\beta(\delta)\sigmabar_n(x),\\\nn
f(x)  \ge \mubar_n(x)-\beta(\delta)\sigmabar_n(x),
\end{eqnarray}
where 
\begin{eqnarray}\nn
\beta(\delta) = \left(\Bigl(2+\frac{\sqrt{\lambda_{\max}}}{\tau}\Bigr)C_k+\frac{1}{\tau} \sqrt{2\max\left\{\xi_0,\frac{2\log(\frac{1}{\delta})}{h_0^2}\right\}\log\Big(\frac{1}{\delta}\Big)}\right),
\end{eqnarray}
and $C_k$, $\xi_0$ and $h_0$ are the parameters specified in Assumption~\ref{ass:f_norm} and Assumption~\ref{ass:light_tailed}. 
\end{theorem}

\begin{proof}
The proof follows the same steps as in the proof of Theorem~\ref{the:conf_subG}. The regression error $f(x)-\mu_n(x)$ is expanded as Terms~$1$ to $4$ in~Equation~\eqref{eq:errorexpression}. The first three terms are bounded exactly as before. The only difference is in bounding Term~$4$, $V^{\top}_n(x)E_n$, which captures the observation noise. The modified is given in the following lemma.
\begin{lemma}\label{lemma:lt}
Under Assumption~\ref{ass:light_tailed}, when the vector $V_n(x)$ is independent of $E_n$, the following equations each hold with probability at least $1-\delta$ for any fixed $x \in \Xc$:
\begin{eqnarray}\nn
V^{\top}_n(x)E_n &\le& \frac{\sigmabar_n(x)}{\tau} \sqrt{2\max\left\{\xi_0,\frac{2\log(\frac{1}{\delta})}{h_0^2}\right\}\log\Big(\frac{1}{\delta}\Big)},\\\nn
V^{\top}_n(x)E_n &\ge& - \frac{\sigmabar_n(x)}{\tau} \sqrt{2\max\left\{\xi_0,\frac{2\log(\frac{1}{\delta})}{h_0^2}\right\}\log\Big(\frac{1}{\delta}\Big)}.
\end{eqnarray}
\end{lemma}
Combining Lemma~\ref{lemma:lt} with the bound on the first three terms in the expression of $f(x)-\mubar_n(x)$, the proof of Theorem~\ref{the:conf_lt} is complete.

\end{proof}

\section{Proof of Corollary~\ref{cor:conf_cont} (Uniform Confidence Intervals on Continuous Domains)}\label{app:cor:conf_cont}

To extend the confidence interval given in Theorem~\ref{the:conf_subG} to hold uniformly on $\Xc$, we use a discretization argument. For this purpose, we apply Assumption~\ref{ass:disc} to $f$ and $\mubar_n$, and also use Assumption~\ref{ass:disc} to bound the discretization error in $\sigmabar_n$.

In the following lemma, we establish a high-probability bound on $\|\mubar_n\|_{\Hc_k}$. 

\begin{lemma}\label{lemma:norm_mu_n}
Under Assumptions~\ref{ass:f_norm} and~\ref{ass:noise}, the RKHS norm of $\mubar_n=V^{\top}_k(x)Y_n$ satisfies the following with probability at least $1-\delta$:
\begin{eqnarray}\nn
\|\mubar_n\|_{\Hc_k}\le C_k\Big(1+\frac{\sqrt{n}k_{\max}}{\tau}\Big)+\frac{\sqrt{n}R}{\tau}\sqrt{2\log\Big(\frac{2n}{\delta}\Big)}.
\end{eqnarray}
\end{lemma}
Let $C^{\mubar_n}_k(\delta) = C_k(1+\frac{\sqrt{n}k_{\max}}{\tau})+\frac{\sqrt{n}R}{\tau}\sqrt{2\log(\frac{2n}{\delta})}$ denote the $1-\delta$ upper confidence bound on $\|\mubar_n\|_{\Hc_k}$, and define $\Ec=\{\|\mubar_n\|_{\Hc_k}\le C_k^{\mubar_n}(\frac{\delta}{2})\}$. By Lemma~\ref{lemma:norm_mu_n}, we have $\Pr[\Ec]\ge 1-\frac{\delta}{2}$. Let $\Xx$ be the discretization of $\Xc$ specified in Assumption~\ref{ass:disc} with RKHS norm bound $C_k^{\mubar_n}(\frac{\delta}{2})$. That is, for any $g\in \Hc_k$ with $\|g\|_{\Hc_k}\le C_k^{\mubar_n}(\frac{\delta}{2})$, we have $g(x) - g([x])\le \frac{1}{n}$, where $[x] = {\arg}{\min}_{ x'\in \Xx}||x'-x||_{l^2}$ is the closest point in $\Xx$ to $x$, and $|\Xx|\le \Gamma_n$, where $\Gamma_n=c\left(C_k^{\mubar_n}(\frac{\delta}{2})\right)^dn^{d}$.

Applying Assumption~\ref{ass:disc} to $f$ and $\mubar_n$ with this discretization, it holds for all $x \in\Xc$ that

\begin{eqnarray}\nn
|f(x)-f([x])|\le \frac{1}{n}.
\end{eqnarray}
In addition, under $\Ec$, for all $x\in\Xc$
\begin{eqnarray}\nn
|\mubar_n(x)-\mubar_n([x])|\le \frac{1}{n}.
\end{eqnarray}

Furthermore, we have the following lemma, which can roughly be viewed as a Lipschitz continuity property for $\sigmabar_n(\cdot)$.

\begin{lemma}\label{lemma:variance_disc}
Under Assumption~\ref{ass:disc}, with the discretization $\Xx$ described above, it holds for all $x \in\Xc$ that
\begin{eqnarray}\nn
\sigmabar_n(x)-\sigmabar_n([x])\le \frac{2}{\sqrt{n}}.
\end{eqnarray}
\end{lemma}


Using this discretization, we also define the event $\Ec_2=\{\mubar_n(x)-\beta(\frac{\delta}{2\Gamma_n})\sigmabar_{n}(x)\le f(x)
, \forall x\in\Xx\}$. By Theorem~\ref{the:conf_subG} and a probability union bound over $x\in\Xx$, we have $\Pr[\Ec_2]\ge 1-\frac{\delta}{2}$.

Accounting for the discretization errors given above, we have under $\Ec_1\cap\Ec_2$ (which holds with probability at least $1-\delta$) that the following lower bound holds for all $x\in\Xc$:
\begin{eqnarray}\nn
f(x) \ge \mubar_n(x)-\frac{2}{n}-\betat_n(\delta)\Big(\sigmabar_n(x)+\frac{2}{\sqrt{n}}\Big),
\end{eqnarray}
where $\betat_n(\delta) = \beta(\frac{\delta}{2\Gamma_n})$. The uniform upper bound holds similarly.


\section{Proofs of Theorems~\ref{the:regression} and~\ref{the:regressioncont} (Error Bounds for Regression)}\label{app:the:regression}

We will prove these theorems using Theorem~\ref{the:conf_subG}, a bound on the ratio between the approximate and true posterior variances, and a bound on the sum of conditional exact variances based on $\gamma_k(n)$. 
The following lemma states that the ratio between posterior variances can be bounded by a constant depending on $\lambda_{\max}$.

\begin{lemma}\label{lemma:ratio}
For $\sigma^2_n$ and $\sigmabar^2_n$ defined in Section~\ref{sec:prelim}, we have for all $x \in \Xc$ that
\begin{eqnarray}\nn
\sigma^2_n(x)&\le& \sigmabar^2_n(x),\\\nn
\sigmabar^2_n(x)&\le&\Big(1+\frac{\lambda_{\max}}{\tau^2}\Big) \sigma^2_n(x).
\end{eqnarray}
\end{lemma}

We have for all $j\le n$ and $x \in \Xc$ that
\begin{eqnarray}\nn
\sigmabar_{n}(x)&\le&\Big(\sqrt{1+\frac{\lambda_{\max}}{\tau^2}}\Big) \sigma_n(x)\\\nn
&\le&\Big(\sqrt{1+\frac{\lambda_{\max}}{\tau^2}}\Big) \sigma_{j-1}(x)\\\nn
&\le&\Big(\sqrt{1+\frac{\lambda_{\max}}{\tau^2}}\Big) \sigmabar_{j-1}(x)\\\nn
&\le&\Big(\sqrt{1+\frac{\lambda_{\max}}{\tau^2}}\Big) \sigmabar_{j-1}(x_j)\\\nn
&\le&\Big(1+\frac{\lambda_{\max}}{\tau^2}\Big) \sigma_{j-1}(x_j),\nn
\end{eqnarray}
where the first, third, and final inequalities all follow from Lemma~\ref{lemma:ratio}, the second inequality holds since conditioning on a larger set reduces the posterior variance of an exact GP model  (which follows from positive-definiteness of the kernel matrix), and the fourth inequality follows from the data selection rule $x_j= \arg\max_{x\in\Xc}\sigmabar_{j-1}(x)$.

Averaging both sides over $j=1,\dots, n$, we have
\begin{eqnarray}
\sigmabar^2_n(x)\le \frac{1}{n}\Big(1+\frac{\lambda_{\max}}{\tau^2}\Big)^2\sum_{j=1}^n \sigma^2_{j-1}(x_j).
\end{eqnarray}
It is well known from \cite{srinivas2010gaussian} that the quantity $\sum_{j=1}^n \sigma^2_{j-1}(x_j)$ can be bounded in terms of the mutual information given in \eqref{eq:mutual_info} for exact GP models:
\begin{eqnarray}
\sum_{j=1}^n \sigma^2_{j-1}(x_j) \le \frac{2}{\log(1+1/\tau^2)}\Ic(Y_n,F).
\end{eqnarray}

Using this bound and the definition of $\gamma_k(n)$ in \eqref{eq:gamma}, we obtain
\begin{eqnarray}\label{eq:bound_sigmabar}
\sigmabar^2_n(x)\le \frac{2(1+\frac{\lambda_{\max}}{\tau^2})^2\gamma_k(n)}{\log(1+1/\tau^2)n}.
\end{eqnarray}

Define the event $\Ec_1=\{\mubar_n(x)-\beta(\frac{\delta}{4|\Xc|})\sigmabar_{n}(x)\le f(x)\le \mubar_n(x)+\beta(\frac{\delta}{4|\Xc|})\sigmabar_{n}(x), \forall x\in\Xc\}$. By Theorem~\ref{the:conf_subG} and a probability union bound over $x$, we have $\Pr[\Ec_1]\ge 1-\frac{\delta}{2}$. In addition, let $\Ec_2 = \{\lambda^{(j)}_{\max}\le \tau^2, \forall j\le n\}$. By Lemma~\ref{Lemma:recRLS}, and a probability union bound over $j$, we have $\Pr[\Ec_2]\ge 1-\frac{\delta}{2}$. We thus condition the rest of the proof on $\Ec_1 \cap \Ec_2$, which holds true with probability at least $1-\delta$. 

We have for all $x\in\Xc$ that
\begin{eqnarray}\nn
|f(x)-\mubar_n(x)|&\le& \beta\Big(\frac{\delta}{4|\Xc|}\Big)\sigmabar_n(x)\\\nn
&&\hspace{-6em}\le \bigg(3C_k+\frac{R}{\tau}\sqrt{2\log\Big(\frac{4|\Xc|}{\delta}\Big)}\bigg)\sigmabar_n(x)\\\label{eq:bound_f_mu_n}
&&\hspace{-6em}\le
\bigg(3C_k+\frac{R}{\tau}\sqrt{2\log\Big(\frac{4|\Xc|}{\delta}\Big)}\bigg)\sqrt{\frac{8\gamma_k(n)}{\log(1+1/\tau^2)n}},
\end{eqnarray}
where the first and second lines hold by $\Ec_1$ and $\Ec_2$, and the third line comes from substituting the bound on $\sigmabar_n(x)$.
Thus, we have with probability at least $1-\delta$ that
\begin{eqnarray}
||f-\mubar_n||_{L^{\infty}} \le
\bigg(3C_k+\frac{R}{\tau}\sqrt{2\log\Big(\frac{4|\Xc|}{\delta}\Big)}\bigg)\sqrt{\frac{8\gamma_k(n)}{\log(1+1/\tau^2)n}},
\end{eqnarray}
which completes the proof of Theorem~\ref{the:regression}.

\paragraph{Extension to continuous domains.} 

Theorem~\ref{the:regressioncont} is proved following the same steps as in the proof of Theorem~\ref{the:conf_subG}. Let us define the event \begin{eqnarray*}
\Ec'_1 =\left\{
\mubar_n(x)-\frac{2}{{n}}-\betat_n(\frac{\delta}{4})\Big(\sigmabar_n(x)+\frac{2}{\sqrt{n}}\Big)\le
f(x) \le \mubar_n(x)+\frac{2}{n}+\betat_n(\frac{\delta}{4})\Big(\sigmabar_n(x)+\frac{2}{\sqrt{n}}\Big), \forall x\in\Xc\right\} . 
\end{eqnarray*}
By Corollary~\ref{cor:conf_cont}, $\Pr[\Ec'_1]\ge 1-\frac{\delta}{2}$.   Recalling also that $\Ec_2 = \{\lambda^{(j)}_{\max}\le \tau^2, \forall j\le n\}$ and $\Pr[\Ec_2]\ge 1-\frac{\delta}{2}$, we can condition on $\Ec'_1\cap\Ec_2$, which holds with probability at least $1-\delta$.

We have for all $x\in\Xc$ that 
\begin{eqnarray}\nn
|f(x)-\mubar_n(x)|&\le& \betat_n(\frac{\delta}{4})\left(\sigmabar_n(x)+\frac{2}{\sqrt{n}}\right)+\frac{2}{n}\\\nn
&&\hspace{-6em}\le\bigg(3C_k+\frac{R}{\tau}\sqrt{2\log\Big(\frac{8\Gamma_n}{\delta}\Big)}\bigg)\left(\sigmabar_n(x)+\frac{2}{\sqrt{n}}\right)+\frac{2}{n}
\\\label{eq:bound_f_mu_n2}
&&\hspace{-6em}\le
\bigg(3C_k+\frac{R}{\tau}\sqrt{2\log\Big(\frac{8\Gamma_n}{\delta}\Big)}\bigg)\left(\sqrt{\frac{8\gamma_k(n)}{\log(1+1/\tau^2)n}}+\frac{2}{\sqrt{n}}\right)+\frac{2}{n},
\end{eqnarray}

where the first and second lines hold by $\Ec_1$ and $\Ec_2$ (as well as $\betat_n(\delta) = \beta(\frac{\delta}{2\Gamma_n})$), and the third line comes from substituing the bound on $\sigmabar_n(x)$ from~\eqref{eq:bound_sigmabar}.
Recall that
$\Gamma_n=  c  \left(C^{\mubar_{n}}_k(\frac{\delta}{2})\right)^dn^{d}$ and $C^{\mubar_n}_k(\delta)=C_k(1+\frac{\sqrt{n}k_{\max}}{\tau})+\frac{\sqrt{n}R}{\tau}\sqrt{2\log(\frac{2n}{\delta})}$. This gives $\Gamma_n=\Oct(C_k^dn^{3d/2})$, and we arrive at
\begin{eqnarray}
\|f-\bar{\mu}_n\|_{L^{\infty}}  = \Oc\left(\sqrt{\frac{d\gamma_k(n)}{n}\log\Big(\frac{n}{\delta}\Big)}\right),
\end{eqnarray}
which completes the proof.




\section{Proofs of Theorems~\ref{the:optimization} and~\ref{the:optimizationcont} (Regret Bounds for Optimization)}\label{app:the:optimization}

Let $t_i=\sum_{j=1}^iN_i$ denote the time at the end of batch $i$, and $\Rc_i=\sum_{j=t_{i-1}+1}^{t_i}(f(x^*)-f(x_j))$ denote the regret incurred in batch $i$. We thus have $\Rc(N)=\sum_{i=1}^B\Rc_i$.

Let $\Ec_1=\{\forall i, \forall x\in\Xc, l_i(x)\le f(x)\le u_i(x)\}$, for $l_i,u_i$ given in~\eqref{eq:conf_batch}, be the event that all the confidence intervals at the end of all batches are satisfied.
By Theorem~\ref{the:conf_subG} and a probability union bound, we have $\Pr[\Ec_1]\ge1-\frac{\delta}{2}$. In addition, let $\Ec_2=\{\forall i,j,\lambda_{\max}^{(i,j)}\le \tau^2\}$. By Lemma~\ref{Lemma:recRLS}, and a probability union bound over $i$ and $j$, we have $\Pr[\Ec_2]\ge 1-\frac{\delta}{2}$. We thus condition the rest of the proof on $\Ec_1\cap\Ec_2$, which holds with probability at least $1-\delta$. 


Under $\Ec_1$, we have for all $x\in\Xc$ that
\begin{eqnarray}\nn
u_{i}(x^*) &\ge& f(x^*)\\\nn
&\ge& f(x)\\\nn
&\ge& l_i(x).
\end{eqnarray}
Thus, $u_i(x^*)\ge\max_{x\in\Xc_i} l_i(x)$ for all $i$, and $x^*$ will not be eliminated by S-BPE; i.e., $x^*\in\Xc_i$ for all $i$. 

The maximum regret in the first batch is bounded as $\Rc_1\le\sqrt{N}(\max_{x\in\Xc}(f(x^*)-f(x)))$.  Moreover, by the reproducing property, $f(x)=\langle f(\cdot),k(x,\cdot)\rangle\le \|f\|_{\Hc_k}k(x,x)\le C_kk_{\max}$. Thus, we have $\Rc_1\le C_kk_{\max}\sqrt{N}$.

For $n>t_{i-1}$ and $i\ge 2$, the regret can be bounded using a similar approach as in the analysis of Theorem~\ref{the:regression} (see also \cite{li2021gaussian}):
\begin{eqnarray}\nn
f(x^*)-f(x_n) &\le&
u_{i-1}(x^*) - l_{i-1}(x_n)\\\nn
&\le& l_{i-1}(x^*)+2\beta\Big(\frac{\delta}{4B|\Xc|}\Big)\sigmabar_{N_{i-1},i-1}(x^*) - u_{i-1}(x_n)+2\beta\Big(\frac{\delta}{4B|\Xc|}\Big)\sigmabar_{N_{i-1},i-1}(x_n)\\\nn
&\le&2\beta\Big(\frac{\delta}{4B|\Xc|}\Big)\sigmabar_{N_{i-1},i-1}(x^*) +2\beta\Big(\frac{\delta}{4B|\Xc|}\Big)\sigmabar_{N_{i-1},i-1}(x_n) \\\nn
&\le& 4\beta\Big(\frac{\delta}{4B|\Xc|}\Big)\sqrt{\frac{8\gamma_k(N_{i-1})}{\log(1+1/\tau^2)N_{i-1}}},
\end{eqnarray}
where the third line follows from the elimination rule of S-BPE, and the last line is proven following the same steps as in the proof of Equation~\eqref{eq:bound_sigmabar}.

Thus, 
\begin{eqnarray}\nn
\Rc_i&=&\sum_{n=t_{i-1}+1}^{t_i} \big(f(x^*)-f(x_n)\big) \\\nn
&\le& 4N_i\beta\Big(\frac{\delta}{4B|\Xc|}\Big)\sqrt{\frac{8\gamma_k(N_{i-1})}{\log(1+1/\tau^2)N_{i-1}}}\\\nn
&\le& 4\beta\Big(\frac{\delta}{4B|\Xc|}\Big)(\sqrt{N}+1)\sqrt{\frac{8\gamma_k(N_{i-1})}{\log(1+1/\tau^2)}}\\\nn
&\le& 
4\beta\Big(\frac{\delta}{4B|\Xc|}\Big)(\sqrt{N}+1)\sqrt{\frac{8\gamma_k(N)}{\log(1+1/\tau^2)}},
\end{eqnarray}
where the second inequality follows from the choice of $N_i$ in Line \ref{line:Ni} of Algorithm~\ref{Alg:S-PBE}.

The definition of $N_i$ ensures that the number of batches satisfies $B \le \lceil\log\log(N)\rceil+1$ \citep[][Proposition~$1$]{li2021gaussian}.  As a result, we have
\begin{eqnarray}\nn
\Rc(N)&=&\sum_{i=1}^B \Rc_i\\\nn
&\le&  4\beta\Big(\frac{\delta}{4B|\Xc|}\Big)(\sqrt{N}+1)\left(\log\log(N)+2\right)\sqrt{\frac{8\gamma_k(N)}{\log(1+1/\tau^2)}}\\\nn
&\le&
\left(12C_k+\frac{4R}{\tau}\sqrt{2\log(\frac{4|\Xc|(\log\log(N)+2)}{\delta})}\right)
(\sqrt{N}+1)\left(\log\log(N)+2\right)\sqrt{\frac{8\gamma_k(N)}{\log(1+1/\tau^2)}}
\\\nn
&=&\Oc\left(\sqrt{N\gamma_k(N)\log\log(N)\log\left(\frac{|\Xc|\log\log(N)}{\delta}\right)}\right),
\end{eqnarray}
where the third line uses the definition of $\beta(\cdot)$ in Theorem~\ref{the:conf_subG}.  Using $\Oct(\cdot)$ notation, this simplifies to $\Rc(N) = \Oct\big(\sqrt{N\gamma_k(N)\log\big(\frac{|\Xc|}{\delta}\big)}\big)$, which establishes Theorem \ref{the:optimization}.

\paragraph{Extension to continuous domains.} 

The proof of Theorem~\ref{the:optimizationcont} is similar to that of Theorem~\ref{the:optimization}. Define the event $\Ec'_1=\{\forall i, \forall x\in\Xc, l_i(x)\le f(x)\le u_i(x)\}$, for $l_i,u_i$ given in~\eqref{eq:conf_cont}.
By Corollary~\ref{cor:conf_cont}, we have $\Pr[\Ec'_1]\ge 1-\frac{\delta}{2}$. Recalling that $\Ec_2 = \{\forall i, j, \lambda^{(i,j)}_{\max}\le \tau^2\}$ and $\Pr[\Ec_2]\ge 1-\frac{\delta}{2}$, we can condition on $\Ec'_1\cap\Ec_2$, which holds with probability at least $1-\delta$.

Then, following similar steps as in the proof of Theorem~\ref{the:optimization}, we have

\begin{eqnarray}\nn
f(x^*)-f(x_n) 
&\le& 4\betat_{N_{i-1}}\Big(\frac{\delta}{4B}\Big)\left(\sqrt{\frac{8\gamma_k(N_{i-1})}{\log(1+1/\tau^2)N_{i-1}}}+\frac{2}{\sqrt{N_{i-1}}}\right)+\frac{8}{N_{i-1}},
\end{eqnarray}
and consequently, for $i\ge 2$,
\begin{eqnarray}
\Rc_i\le 4\betat_{N}\Big(\frac{\delta}{4B}\Big)\left((\sqrt{N}+1)\sqrt{\frac{8\gamma_k(N)}{\log(1+1/\tau^2)}}+2(\sqrt{N}+1)\right)+8\sqrt{N}.
\end{eqnarray}

Note that $\betat_n(\delta)$ in non-decreasing in $n$.
Summing the regret over batches, we have
\begin{eqnarray}\nn
\Rc(N)&=&\sum_{i=1}^B \Rc_i\\\nn
&\le& 4\betat_{N}\Big(\frac{\delta}{4B}\Big)(\sqrt{N}+1)\left(\log\log(N)+2\right)\left(\sqrt{\frac{8\gamma_k(N)}{\log(1+1/\tau^2)}}+2\right)+8\sqrt{N}(\log\log(N)+2)\\\nn
&\le&
\left(12C_k+\frac{4R}{\tau}\sqrt{2\log\Big(\frac{8\Gamma_N(\log\log(N)+2)}{\delta}\Big)}\right)
(\sqrt{N}+1)\left(\log\log(N)+2\right)\left(\sqrt{\frac{8\gamma_k(N)}{\log(1+1/\tau^2)}}+2\right)\\\nn
&&\hspace{25em}+~8\sqrt{N}(\log\log(N)+2)
\\\nn
&=&\Oc\left(\sqrt{N\gamma_k(N)\log\log(N)\log\left(\frac{\Gamma_N\log\log(N)}{\delta}\right)}\right).
\end{eqnarray}

Recall the notations $\betat_n(\delta) = \beta(\frac{\delta}{2\Gamma_n})$,
$\Gamma_n=  c  \left(C^{\mubar_{n}}_k(\frac{\delta}{2})\right)^dn^{d}$ and $C^{\mubar_n}_k(\delta)=C_k(1+\frac{\sqrt{n}k_{\max}}{\tau})+\frac{\sqrt{n}R}{\tau}\sqrt{2\log(\frac{2n}{\delta})}$.  This yields $\Gamma_n=\Oct(C_k^dn^{3d/2})$, and we arrive at
\begin{eqnarray}\nn
\Rc(N)
&=&\Oct\left(
\sqrt{Nd\gamma_k(N)\log\Big(\frac{N}{\delta}\Big)}
\right),
\end{eqnarray}
which completes the proof of Theorem \ref{the:optimizationcont}.


\section{Proofs of Auxiliary Lemmas}\label{app:proof_lemmas}

In this section, we prove the auxiliary lemmas stated in the preceding sections.

\subsection{Proof of Lemma~\ref{Lemma:projectin_mean} (Expression for the Projection)}

We can express the projection as $\Pi_{\kappa}[f]  = \sum_{j=1}^m\alpha^*_jk(\cdot,z_j)$, where $\alpha^*=[\alpha_1,\dots,\alpha_m]^{\top}$ is the solution of 
\begin{eqnarray}
\min_{\alpha\in\Rr^m}\Big\|f-\sum_{j=1}^m\alpha_jk(\cdot,z_j)\Big\|_{\Hc_{k}}.
\end{eqnarray}
By the reproducing property, we have 
\begin{eqnarray}
\Big\|f(\cdot)-\sum_{j=1}^m\alpha_jk(\cdot,z_j)\Big\|^2_{\Hc_{k}} = \|f\|^2_{\Hc_k} - 2\alpha^{\top} f_{Z_m} + \alpha^{\top} k_{Z_m,Z_m}\alpha.
\end{eqnarray}
Explicitly solving for the minimum of this quadratic expression gives $\alpha^*=k^{-1}_{Z_m,Z_m}f_{Z_m}$, and thus $\Pi_{\kappa}[f]= k^{\top}_{Z_m}(\cdot)k^{-1}_{Z_m,Z_m}f_{Z_m}$, which completes the proof.

\subsection{Proof of Lemma~\ref{lemma:Term2} (Expression for Term 2)}

We show that $\text{Term}~2 = k^{\top}_{Z_m}(\cdot)\left(k_{Z_m,Z_m}+\frac{1}{\tau^2}k^{\top}_{X_n,Z_m}k_{X_n,Z_m}\right)^{-1}k_{Z_m}(\cdot)$ is equal to $\sigma^2_{\kappa,\tau, X_n}(\cdot)$.  We have
\begin{eqnarray}\nn
\text{Term}~ 2 &=&k^{\top}_{Z_m}(\cdot)\Big(k_{Z_m,Z_m}+\frac{1}{\tau^2}k^{\top}_{X_n,Z_m}k_{X_n,Z_m}\Big)^{-1}k_{Z_m}(\cdot)\\\nn
&=& k^{\top}_{Z_m}(\cdot)\Big(k^{-1}_{Z_m,Z_m} -k^{-1}_{Z_m,Z_m}k^{\top}_{X_n,Z_m} \left(k_{X_n,Z_m}k^{-1}_{Z_m,Z_m}k^{\top}_{X_n,Z_m}+\tau^2 I_n\right)^{-1}k_{X_n,Z_m}k^{-1}_{Z_m,Z_m}\Big)k_{Z_m}(\cdot)\\\nn
&=& k^{\top}_{Z_m}(\cdot)k^{-1}_{Z_m,Z_m}k_{Z_m}(\cdot)\\\nn &&-k^{\top}_{Z_m}(\cdot)k^{-1}_{Z_m,Z_m}k_{Z_m,X_n} \left(k^{\top}_{Z_m,X_n}k^{-1}_{Z_m,Z_m}k_{Z_m,X_n}+\tau^2I_n\right)^{-1}k^{\top}_{Z_m,X_n}k^{-1}_{Z_m,Z_m}k_{Z_m}(\cdot)
\\
&=& \kappa(\cdot,\cdot) - \kappa^{\top}_{X_n}(\cdot)(\kappa_{X_n,X_n}+\tau^2I_n)^{-1}\kappa_{X_n}(\cdot)\\\nn
&=& \sigma^2_{\kappa,\tau, X_n}(\cdot),
\end{eqnarray}
where the second line follows form Woodbury matrix identity, and the fourth line follows from the definition of $\kappa(\cdot, \cdot)$ given in Equation \eqref{eq:kappa}.

\subsection{Proof of Lemma~\ref{lemma:V_nkappa} (Expression for $V_n$)}

We have
\begin{eqnarray}\nn
\kappa^{\top}_{X_n}(\cdot)(\tau^2I_n+\kappa_{X_n,X_n})^{-1}
&=&k_{Z_m}(\cdot)k^{-1}_{Z_m,Z_m}k^{\top}_{X_n, Z_m}\left(k_{X_n,Z_m}k^{-1}_{Z_m,Z_m}k^{\top}_{X_n,Z_m}+\tau^2 I_n\right)^{-1}\\\nn
&=& k_{Z_m}(\cdot)k^{-1}_{Z_m,Z_m}\left(k^{\top}_{X_n,Z_m}k_{X_n,Z_m}k^{-1}_{Z_m,Z_m}+\tau^2 I_n\right)^{-1}k^{\top}_{X_n,Z_m} \\\nn
&=& k^{\top}_{Z_m}(\cdot)\left(\tau^2k_{Z_m,Z_m}+k^{\top}_{X_n,Z_m}k_{X_n,Z_m}\right)^{-1} k^{\top}_{X_n,Z_m}\\\nn
&=&V_n,
\end{eqnarray}
where the first line comes from the definition of $\kappa$, the second line uses the push-through matrix identity $A(I+BA)^{-1} = (I+AB)^{-1}A$ with $A=k^{\top}_{X_n,Z_m}$ and $B=k_{X_n,Z_m}k^{-1}_{Z_m,Z_m}$, and the third line uses $C^{-1}D^{-1}=(DC)^{-1}$ with $C=k_{Z_m,Z_m}$ and $D=\tau^2 I_n +k^{\top}_{X_n,Z_m}k_{X_n,Z_m}$.

\subsection{Proof of Lemma~\ref{lemma:chern} (High-Probability Bounds on Term 4)}

For a sub-Gaussian random variable $X$ with parameter $R$, Chernoff-Hoeffding inequality~\citep[e.g., see][]{antonini2008convergence} implies that the following inequalities hold with probability at least $1-\delta$ each (or $1-2\delta$ jointly):
\begin{eqnarray}\nn
X&\le& R\sqrt{2\log\Big(\frac{1}{\delta}\Big)},\\\nn
X&\ge& -R\sqrt{2\log\Big(\frac{1}{\delta}\Big)}.
\end{eqnarray}

When the vector $V_n(x)$ is constructed independent of $E_n$, and under Assumption~\ref{ass:noise}, $V_n^{\top}(x)E_n$ is sub-Gaussian with parameter $R\|V_n(x)\|_{l^2}$. Using Theorem~\ref{the:approximatevariance}, we have   $R\|V_n(x)\|_{l^2}\le \frac{R\sigmabar_n(x)}{\tau}$. Thus, applying Chernoff-Hoeffding inequality to $V_n^{\top}(x)E_n$, we have the following with probability $1-\delta$ each:
\begin{eqnarray}\nn
V^{\top}_n(x)E_n &\le& \frac{R\sigmabar_n(x)}{\tau}\sqrt{2\log\Big(\frac{1}{\delta}\Big)},\\\nn
V^{\top}_n(x)E_n &\ge& - \frac{R\sigmabar_n(x)}{\tau}\sqrt{2\log\Big(\frac{1}{\delta}\Big)}.
\end{eqnarray}

\subsection{Proof of Lemma~\ref{lemma:lt} (High-Probability Bounds on the Noise Term)}

When $V_n(x)$ is independent of $E_n$, and when the noise vector $E_n$ satisfies Assumption~\ref{ass:light_tailed}, following a standard Chernoff-like confidence bound, the following equations each hold with probability at least $1-\delta$~\citep[e.g., see][]{vakili2021optimal}:
\begin{eqnarray}\nn
V^{\top}_n(x)E_n &\le& \|V_n(x)\|_{l^2} \sqrt{2\max\left\{\xi_0,\frac{2\log(\frac{1}{\delta})}{h_0^2}\right\}\log\Big(\frac{1}{\delta}\Big)},\\\nn
V^{\top}_n(x)E_n &\ge& - \|V_n(x)\|_{l^2} \sqrt{2\max\left\{\xi_0,\frac{2\log(\frac{1}{\delta})}{h_0^2}\right\}\log\Big(\frac{1}{\delta}\Big)}.
\end{eqnarray}

Substituting $\|V_n(x)\|_{l^2}\le \frac{\sigmabar_n(x)}{\tau}$ from Theorem~\ref{the:approximatevariance}, we arrive at the lemma.

\subsection{Proof of Lemma~\ref{lemma:norm_mu_n} (RKHS Norm of the Approximate Posterior Mean)}

Recalling that $\mubar_n(\cdot) = V_n^{\top}(\cdot)Y_n$, we have
\begin{eqnarray}\nn
\|\mubar_n(\cdot)\|_{\Hc_k} &=& \|V_n^{\top}(\cdot)Y_n\|_{\Hc_k}\\\nn
&=& \|V_n^{\top}(\cdot)(f_{X_n}+E_n)\|_{\Hc_k}\\\nn
&\le&\underbrace{\|V_n^{\top}(\cdot)\Pi_{\kappa}[f]_{X_n}\|_{\Hc_k}}_{\text{Term}~1} +
\underbrace{\|V_n^{\top}(\cdot)(f_{X_n}-\Pi_{\kappa}[f]_{X_n})\|_{\Hc_k}}_{\text{Term}~2} +
\underbrace{\|V_n^{\top}(\cdot)E_n\|_{\Hc_k}}_{\text{Term}~3}.
\end{eqnarray}
We proceed by bounding these three terms. 

\paragraph{Term~$1$.} Recall the RKHS-based definition of $\mubar_n$:
\begin{eqnarray}\nn
\bar{\mu}_n=\arg\min_{g\in\Hc_{\kappa}} \sum_{i=1}^n(y_i-g(x_i))^2+\tau^2\|g\|^2_{\Hc_\kappa}.
\end{eqnarray}

Consider a noiseless projected dataset $\tilde{\Dc}_n = \{(x_i,\Pi_{\kappa}[f](x_i))_{i=1}^n\}$, and note that $\mubar_{k,\tau, Z_m,X_n,\Pi_{\kappa}[f]_{X_n}}=V^{\top}_n(x)\Pi_{\kappa}[f]_{X_n}$. Moreover, using the RKHS-based definition of $\mubar_{k,\tau, Z_m,X_n,\Pi_{\kappa}[f]_{X_n}}$, we have
\begin{eqnarray}\nn
\sum_{i=1}^n(\Pi_{\kappa}[f](x_i)-V^{\top}_n(x_i)\Pi_{\kappa}[f]_{X_n})^2+\tau^2\|V^{\top}_n(\cdot)\Pi_{\kappa}[f]_{X_n}\|^2_{\Hc_\kappa}&\le& \sum_{i=1}^n(\Pi_{\kappa}[f](x_i)-\Pi_{\kappa}[f](x_i))^2+\tau^2\|\Pi_{\kappa}[f]\|^2_{\Hc_\kappa} \\\nn
&=& \tau^2\|\Pi_{\kappa}[f]\|^2_{\Hc_\kappa}.
\end{eqnarray}
Thus, $\|V^{\top}_n(\cdot)\Pi_{\kappa}[f]_{X_n}\|_{\Hc_\kappa}\le \|\Pi_{\kappa}[f]\|_{\Hc_\kappa}=\|\Pi_{\kappa}[f]\|_{\Hc_k}\le\|f\|_{\Hc_k}\le C_k$.  

\paragraph{Term $2$.} First note that for any vector $v\in\Rr^{n}$, we have
\begin{eqnarray}\nn
\|V^{\top}_n(\cdot)v\|^2_{\Hc_{\kappa}}&=& 
\langle V^{\top}_n(\cdot)v,V^{\top}_n(\cdot)v\rangle_{\Hc_{\kappa}}\\\nn
&=& 
\langle \kappa^{\top}_{X_n}(\cdot)(\tau^2I_n+\kappa_{X_n,X_n})^{-1}v,\kappa^{\top}_{X_n}(\cdot)(\tau^2I_n+\kappa_{X_n,X_n})^{-1}v\rangle_{\Hc_{\kappa}}\\\nn
&=&v^{\top}(\kappa_{X_n,X_n}+\tau^2 I_n)^{-1}\kappa_{X_n,X_n}(\kappa_{X_n,X_n}+\tau^2 I_n)^{-1}v\\\nn
&=&v^{\top}(\kappa_{X_n,X_n}+\tau^2 I_n)^{-1}(\kappa_{X_n,X_n}+\tau^2 I_n-\tau^2 I_n)(\kappa_{X_n,X_n}+\tau^2 I_n)^{-1}v\\\nn
&=&v^{\top}(\kappa_{X_n,X_n}+\tau^2 I_n)^{-1}v - \tau^2v^{\top}(\kappa_{X_n,X_n}+\tau^2 I_n)^{-2}v \\\nn
&\le&v^{\top}(\kappa_{X_n,X_n}+\tau^2 I_n)^{-1}v \\
&\le&\frac{\|v\|^2_{l^2}}{\tau^2}, \label{eq:V_norm}
\end{eqnarray}
where the second line uses $V^{\top}_n(\cdot) = \kappa^{\top}_{X_n}(\cdot)(\tau^2I_n+\kappa_{X_n,X_n})^{-1}$ from Lemma~\ref{lemma:V_nkappa}, and the third line follows from the reproducing property of the RKHS. Thus, 
\begin{eqnarray}\nn
\|V_n^{\top}(\cdot)(f_{X_n}-\Pi_{\kappa}[f]_{X_n})\|_{\Hc_k}&\le& \frac{\|f_{X_n}-\Pi_{\kappa}[f]_{X_n}\|_{l^2}}{\tau}\\\nn
&\le& \frac{\sqrt{n}\max_{1\le i\le n}(f(x_i)-\Pi_{\kappa}[f](x_i))}{\tau}.
\end{eqnarray}

Recall the notations $f^{\perp}=f-\Pi_{\kappa}[f]$ and $\kappa^{\perp}=k-\kappa$. By the reproducing property of the RKHS, we have
\begin{eqnarray}\nn
f^{\perp}(x)  &=&\langle{f^{\perp}(\cdot),\kappa^{\perp}(x,\cdot)\rangle_{\Hc_{\kappa^{\perp}}}}\\\nn
&\le&\|f^{\perp}\|_{\Hc_{\kappa^{\perp}}}\kappa^{\perp}(x,x)\\\nn
&\le& C_k k_{\max},
\end{eqnarray}
where we used $\|f^{\perp}\|_{\Hc_{\kappa^{\perp}}} = \|f^{\perp}\|_{\Hc_{k}}\le \|f\|_{\Hc_k}\le C_k$, as well as $\kappa^{\perp}(x,x)\le k(x,x)$, which follows from the definition of~$\kappa$. 

Thus, we have
\begin{eqnarray}\nn
\|V_n^{\top}(\cdot)(f_{X_n}-\Pi_{\kappa}[f]_{X_n})\|_{\Hc_k}\le \frac{\sqrt{n}C_k k_{\max}}{\tau}. 
\end{eqnarray}

\paragraph{Term~$3$.} For the third term, using $\|V^{\top}_n(\cdot)v\|_{\Hc_k}\le\frac{\|v\|_{l^2}}{\tau}$ as established in \eqref{eq:V_norm}, we have
\begin{eqnarray}\nn
\|V^{\top}_n(\cdot)E_n\|^2_{\Hc_k}\le\frac{\|E_n\|^2_{l^2}}{\tau^2}. 
\end{eqnarray}

Under Assumption~\ref{ass:noise}, as a result of Chernoff-Hoeffding inequality, we have for each $i \in \{1,\dotsc,n\}$ that
\begin{eqnarray}\nn
\epsilon_i^2\le 2R^2\log\Big(\frac{2}{\delta}\Big) 
\end{eqnarray}
with probability at least $1-\delta$.  Hence, using a probability union bound, we obtain
\begin{eqnarray}\nn
\frac{\|E_n\|^2_{l^2}}{\tau^2}\le \frac{2nR^2}{\tau^2}\log\Big(\frac{2n}{\delta}\Big)
\end{eqnarray}
with probability at least $1-\delta$, which in turn gives
\begin{eqnarray}\nn
\|V^{\top}_n(\cdot)E_n\|_{\Hc_k}\le\frac{\sqrt{n}R}{\tau}\sqrt{2\log\Big(\frac{2n}{\delta}\Big)}. 
\end{eqnarray}

Putting the bounds on Terms~$1$ to $3$ together, we arrive at the lemma.

\subsection{Proof of Lemma~\ref{lemma:variance_disc} (Discretization Error in the Posterior Standard Deviation)}

Recall the expression
\begin{eqnarray}\nn
\sigmabar^2_n(\cdot,\cdot) &=& \underbrace{k(\cdot,\cdot) - k^{\top}_{Z_m}(\cdot)k^{-1}_{Z_m,Z_m}k_{Z_m}(\cdot)}_{\text{Term}~1}+ \underbrace{k^{\top}_{Z_m}(\cdot)\left(k_{Z_m,Z_m}+\frac{1}{\tau^2}k^{\top}_{X_n,Z_m}k_{X_n,Z_m}\right)^{-1}k_{Z_m}(\cdot)}_{\text{Term}~2},
\end{eqnarray}
from the proof of Theorem~\ref{the:approximatevariance}, and recall that Term $2 = \sigma^2_{\kappa,\tau,X_n}(\cdot)$ from Lemma~\ref{lemma:Term2}.  By the definition $\kappa(\cdot, \cdot') = k^{\top}_{Z_m}(\cdot)k^{-1}_{Z_m,Z_m}k_{Z_m}(\cdot')$, it follows that
\begin{eqnarray*}
\sigmabar^2_n(\cdot,\cdot) = k(\cdot,\cdot) - \kappa^{\top}_{X_n}(\cdot)(\kappa_{X_n,X_n}+\tau^2I_n)^{-1}\kappa_{X_n}(\cdot).
\end{eqnarray*}

By reproducing property, we have for all $x\in\Xc$ that $\|k(x,.)\|_{\Hc_k}=k(x,x)\le k_{\max}$, where $k_{\max}=\sup_{x\in\Xc}k(x,x)$. 
We also have for all $x\in\Xc$ that
\begin{eqnarray*}
\left\|\kappa^{\top}_{X_n}(x)(\kappa_{X_n,X_n}+\tau^2I_n)^{-1}\kappa_{X_n}(\cdot)\right\|^2_{\Hc_k}  &=& \langle \kappa^{\top}_{X_n}(x)(\kappa_{X_n,X_n}+\tau^2I_n)^{-1}\kappa_{X_n}(\cdot),\kappa^{\top}_{X_n}(x)(\kappa_{X_n,X_n}+\tau^2I_n)^{-1}\kappa_{X_n}(\cdot)\rangle\\
&=& \kappa^{\top}_{X_n}(x)(\kappa_{X_n,X_n}+\tau^2I_n)^{-1}\kappa_{X_n,X_n}(\kappa_{X_n,X_n}+\tau^2I_n)^{-1}\kappa_{X_n}(x)\\
&=&\kappa^{\top}_{X_n}(x)(\kappa_{X_n,X_n}+\tau^2I_n)^{-1}\kappa_{X_n}(x)-\kappa^{\top}_{X_n}(x)(\kappa_{X_n,X_n}+\tau^2I_n)^{-2}\kappa_{X_n}(x)\\
&\le&\frac{1}{\tau^2}\|\kappa_{X_n}(x)\|^2\\
&\le&\frac{nk^2_{\max}}{\tau^2},
\end{eqnarray*}
by similar steps to those leading up to \eqref{eq:V_norm}.
Thus, $\left\|\kappa^{\top}_{X_n}(x)(\kappa_{X_n,X_n}+\tau^2I_n)^{-1}\kappa_{X_n}(\cdot)\right\|_{\Hc_k}\le \frac{\sqrt{n}k_{\max}}{\tau}$.

In the following, define the function $q(\cdot,\cdot')=\kappa^{\top}_{X_n}(\cdot)(\kappa_{X_n,X_n}+\tau^2I_n)^{-1}\kappa_{X_n}(\cdot')$.
For all $x\in\Xc$, applying Assumption~\ref{ass:disc} to $k(x,.)$ and $q(x,.)$, with $\Xx$, we have $|k(x,x')-k(x,[x'])|\le\frac{1}{n}$ and $|q(x,x')-q(x,[x'])|\le\frac{1}{n}$. Thus,
\begin{eqnarray*}\nn
|\sigmabar^2_n(x) - \sigmabar^2([x]) |
&=&\big|(k(x,x)-q(x,x))-(k([x],[x])-q([x],[x]))\big|\\
&=&\big|(k(x,x)-q(x,x)) -(k(x,[x])-q(x,[x]))
+(k(x,[x])-q(x,[x]))
-(k([x],[x])-q([x],[x]))\big|\\\nn
&\le& |k(x,x)-k(x,[x])| + |k(x,[x])
-k([x],[x])| +|q(x,x)-q(x,[x])| - |q(x,[x])
-q([x],[x])|\\\nn
&\le&\frac{4}{n}.
\end{eqnarray*}

To obtain a discretization error bound for the standard deviation from that of the variance, we write
\begin{eqnarray}\nn
(\sigmabar_n(x) - \sigmabar([x]))^2&\le& 
\left|\sigmabar_n(x) - \sigmabar([x])\right|(\sigmabar_n(x) + \sigmabar([x]))\\\nn
&=&|\sigmabar^2_n(x) - \sigmabar^2([x])|\\\nn
&\le&\frac{4}{n}.
\end{eqnarray}

Therefore,
\begin{eqnarray*}
|\sigmabar_n(x) - \sigmabar([x])|\le\frac{2}{\sqrt{n}}.
\end{eqnarray*}

\subsection{Proof of Lemma~\ref{lemma:ratio} (Bounds on the Approximate Posterior Variance)}


For this proof, we introduce a feature map $\phi(\cdot)$ (for example, the Mercer feature vector) such that $k(x,x')=\phi^{\top}(x)\phi(x')$. Similarly, we introduce $\psi(\cdot)$ such that $\kappa(x,x')=\psi^{\top}(x)\psi(x')$. Let $\Phi = [\phi^{\top}(x_1), ...,\phi^{\top}(x_n)]^{\top}$ be the feature matrix at observation points, and similarly for $\Psi$. Note that $k_{X_n,X_n} = \Phi\Phi^{\top}$ and $\kappa_{X_n,X_n} = \Psi\Psi^{\top}$. For this proof, we use the feature space forms of the posterior variances~\citep[e.g., see,][]{Calandriello2019Adaptive}: $\sigma_n^{2}(x)=\phi^{\top}(x)(\Phi^{\top}\Phi+\tau^2I)^{-1}\phi(x)$, and $\sigmabar_n^{2}(x)=\phi^{\top}(x)(\Psi^{\top}\Psi+\tau^2I)^{-1}\phi(x)$  .


With these definitions in place, the definition of $\lambda_{\max}$ can equivalently be expressed as
\begin{eqnarray}\nn
\Phi^{\top}\Phi - \Psi^{\top}\Psi \preccurlyeq \lambda_{\max} I,
\end{eqnarray}
where $I$ is the identity operator, and the notation $A\preccurlyeq B$ means $B-A$ is positive semi-definite.  Thus, 
\begin{eqnarray}\nn
\Phi^{\top}\Phi+\tau^2I &\preccurlyeq& \Psi^{\top}\Psi+\tau^2I+\lambda_{\max}I\\\nn
&\preccurlyeq&
\Psi^{\top}\Psi+\tau^2I + 
\frac{\lambda_{\max}}{\tau^2}(\Psi^{\top}\Psi+\tau^2I)\\\nn
&=&\Big(1+\frac{\lambda_{\max}}{\tau^2}\Big)(\Psi^{\top}\Psi+\tau^2I),
\end{eqnarray}
and rearranging gives
\begin{eqnarray}\nn
(\Psi^{\top}\Psi+\tau^2I)^{-1} \preccurlyeq \Big(1+\frac{\lambda_{\max}}{\tau^2}\Big)(\Phi^{\top}\Phi+\tau^2I)^{-1}.
\end{eqnarray}

Thus, by the above equations for posterior variances in the feature space representation, we obtain
\begin{eqnarray}\nn
\sigmabar_n^2(x)\le \Big(1+\frac{\lambda_{\max}}{\tau^2}\Big)\sigma_n^2(x).
\end{eqnarray}

The remaining inequality in Lemma \ref{lemma:ratio} follows similarly from $\Psi^{\top}\Psi\preccurlyeq\Phi^{\top}\Phi$.


\end{document}